\documentclass{article}
\usepackage[utf8]{inputenc}

\usepackage[margin=1in]{geometry}

\usepackage{float,graphicx,verbatim,fullpage,hyperref,amssymb,amsmath,amsthm,enumerate,multicol, xspace,xcolor,mathtools,thmtools,thm-restate,cleveref,xspace,tikz,caption,subcaption,wasysym, enumitem}
\usepackage[margin=1in]{geometry}
\usepackage{algorithm}
\usepackage[noend]{algpseudocode}
\usepackage{enumitem}
\usepackage{comment}
\definecolor{ForestGreen}{RGB}{34,139,34}

\usetikzlibrary{calc}
\usetikzlibrary{arrows.meta}
\tikzset{>={Latex[width=1.5mm,length=1.5mm]}}

\usepackage{mleftright}
\usepackage{hyperref}
    \hypersetup{colorlinks=true, linkcolor=blue, filecolor=magenta, urlcolor=blue, citecolor=red}

\usepackage{tikz}
\tikzstyle{vertex}=[circle, draw, inner sep=1pt, minimum size=15pt]
\usetikzlibrary{arrows.meta}
\tikzset{>={Latex[width=1.5mm,length=1.5mm]}}

\newfloat{procedure}{htbp}{loa}
\floatname{procedure}{Procedure}

\def\R{\mathbb{R}}
\def\E{\mathbb{E}}

\newcommand{\cE}{\mathcal{E}}

\newcommand{\cS}{\mathcal{S}}
\newcommand{\cD}{\mathcal{D}}

\def\ep{\varepsilon}

\newtheorem{theorem}{Theorem}[section]
\newtheorem{proposition}[theorem]{Proposition}

\newtheorem{lemma}[theorem]{Lemma}
\newtheorem{claim}[theorem]{Claim}

\newtheorem{example}{Example}

\theoremstyle{definition}

\usepackage{amsmath,thm-restate}

\def\E{\mathbb{E}}
\def\R{\mathbb{R}}

\newcommand{\Exp}{{\normalfont\text{Exp}}}
\newcommand{\Gum}{{\normalfont\text{Gum}}}

\newcommand{\cF}{\mathcal{F}}
\newcommand{\cI}{\mathcal{I}}

\newcommand{\cus}{\textsc{Cus}\xspace}
\newcommand{\hyb}{\textsc{Hyb}\xspace}
\newcommand{\mnl}{\textsc{Mnl}\xspace}
\newcommand{\lcmnl}{\textsc{Lc-Mnl}\xspace}
\newcommand{\pomnl}{\textsc{Pomnl}\xspace}
\newcommand{\lcpomnl}{\textsc{Lc-Pomnl}\xspace}
\newcommand{\mc}{\textsc{Mc}\xspace}
\newcommand{\cusmc}{\textsc{Cus-Mc}\xspace}
\newcommand{\hybmc}{\textsc{Hyb-Mc}\xspace}

\usepackage[table]{xcolor}
\usepackage{pgfplotstable}
\usepackage{colortbl}
\usepackage{pgfplots}
\usepackage[dvipsnames]{xcolor}

\usepackage{hyperref}
\hypersetup{colorlinks=true, linkcolor=blue, filecolor=magenta, urlcolor=blue, citecolor=blue}

\pgfplotsset{compat=1.18}

\pgfplotstableset{
    /color cells/min/.initial=0,
    /color cells/max/.initial=1.15,
    color cells/.style={
        col sep=comma,
        string type,
        postproc cell content/.code={%
                \pgfkeysalso{@cell content=\rule{0cm}{2.4ex}\cellcolor{Fuchsia!##1}\pgfmathtruncatemacro\number{##1}\ifnum\number>0\color{white}\fi##1}%
                },
        columns/x/.style={
            column name={},
            postproc cell content/.code={}
        }
    }
}

\usetikzlibrary{arrows.meta}
\tikzset{>={Latex[width=1.5mm,length=1.5mm]}}
\tikzstyle{vertex}=[circle, draw, inner sep=1pt, minimum size=14pt]

\def\final{0}  
\def\iflong{\iffalse}
\ifnum\final=0  
\newcommand{\ynote}[1]{{\color{blue}[{\small Young-San: \bf #1}]\marginpar{\color{red}*}}}
\newcommand{\gnote}[1]{{\color{red}[{\small Gerardo: \bf #1}]\marginpar{\color{red}*}}}
\newcommand{\yanote}[1]{{\color{blue}[{\small Yalcin: \bf #1}]\marginpar{\color{red}*}}}
\newcommand{\todo}[1]{{\color{red}[{ TODO: \bf #1}]\marginpar{\color{red}*}}}
\else 
\newcommand{\ynote}[1]{}
\newcommand{\gnote}[1]{}
\newcommand{\yanote}[1]{}
\newcommand{\todo}[1]{}
\fi

\DeclareMathOperator{\poly}{poly}

\usepackage{bbm}
\usepackage{natbib}

\title{Estimation, Prediction, and Assortment Optimization for Markov Chain Choice Models with Panel Data}
\author{Yalcin Akcay\thanks{Melbourne Business School, Email: \{y.akcay, g.berbeglia, y.lin\}@mbs.edu}, Gerardo Berbeglia\footnotemark[1], Young-San Lin\footnotemark[1]}
\date{\today}

\begin{document}

\maketitle

\begin{abstract}
    We propose a framework for the Markov chain (MC) choice model with panel data, including parameter estimation, personalized choice prediction, and personalized assortment optimization. In contrast to the traditional setting, which assumes that each transaction is independently drawn from a random utility model, our framework accounts for dependencies among transactions for the same customer in historical data, captured by partial-ordering preference information. To the best of our knowledge, our framework initiates the study of choice modeling with panel data under MC. As our primary result, we propose novel expectation-maximization (EM) algorithms for MC parameter estimation by incorporating partial-ordering-based customer preference information. On synthetic datasets and the sushi dataset, our EM algorithms outperform the traditional EM algorithm of {\c{S}}im{\c{s}}ek and Topaloglu (Operations Research, 66, 2018) and multinomial-logit-based partial-order benchmarks adapted from Jagabathula and Vulcano (Management Science, 64, 2018). As our secondary contribution, we present hardness and computational results for conditional choice prediction and assortment optimization problems. These results complement our estimation framework and clarify the computational landscape of conditional choice and assortment optimization, which may be of independent interest.
\end{abstract}

\section{Introduction} \label{sec:intro}





Incorporating customer choice behavior into assortment planning is a central problem in revenue management. A standard roadmap for assortment planning consists of two key steps: choice model estimation and assortment optimization. In the literature, traditional choice model estimation algorithms typically assume that transactions in historical datasets are independent, while the assortment optimization problem seeks to maximize the expected revenue aggregated over the entire customer population. A crucial feature of these traditional frameworks is their reliance on unconditional choice probabilities—that is, choice probabilities are assumed to be identical across customers and independent of customer-specific transaction histories.

Recent advances in data collection and e-commerce technologies now allow retailers to observe detailed transaction histories at the individual customer level. This information, popularly referred to as \emph{panel data}, reveals correlations among transactions made by the same customer, challenging the independence assumptions underlying classical models. This development naturally motivates the following questions:
\begin{center}
\emph{Can customer choice probabilities be estimated more precisely by incorporating individual historical transaction information into a choice model? What is the computational effort needed for this task?}
\end{center}
These questions take the choice model as given and focus on personalization at the level of choice probability estimation. Specifically, we study how to compute customer-specific choice probabilities conditional on individual transaction histories, while assuming that the unconditional choice probabilities are governed by an underlying choice model that is common to all customers.

Intuitively, transaction histories provide additional information about individual preferences, as repeated choices made by the same customer are inherently correlated. From a broader perspective, this raises a second complementary question.
\begin{center}
\emph{Can the estimation of the choice model for the entire population improve by exploiting individual historical transaction information? When is this approach more beneficial?}
\end{center}
Here, the emphasis shifts from prediction to inference, asking whether exploiting information from repeated customer choice interactions can lead to more accurate estimation of the underlying choice model.

Motivated by these questions, we develop both theoretical and experimental results on choice model estimation, conditional choice prediction, and assortment optimization with panel data. Our analysis demonstrates how individual-level transaction data can be systematically leveraged to enhance both predictive accuracy and decision-making performance in assortment planning.

\subsection{Our Results}
To the best of our knowledge, our framework initiates the study of choice modeling with panel data under the Markov chain (MC) choice model. Our results are summarized as follows.
\begin{itemize}
    \item As a motivating question, we investigate for which choice models panel data improve population-level estimation. The primary goal is to maximize the \emph{likelihood} of the given data by selecting the best parameters. Here, likelihood is defined as the probability that the full set of observations occurs independently. In the traditional setting, each observation is a transaction. In the setting with panel data, each observation is a customer-specific partial ordering described by customer-specific historical transactions. Under multinomial logit (MNL), where the choice probabilities are observed in the population limit, using panel data does not add identification power for the MNL parameters. In contrast, for MC, customer-level information can improve model estimation. This result highlights the gap between MNL and MC regarding how exploiting individual transaction information could facilitate choice model estimation, thereby motivating the use of panel data when estimating an MC.
    \item As our primary result, we present expectation-maximization (EM) algorithms that incorporate customer-level information to estimate an MC. The EM algorithm \cus assumes that each customer has a strict underlying preference, and aims to maximize the likelihood that each customer has a specific transaction set, which is equivalently captured by a partial-ordering preference \cite{jagabathula2018partial,jagabathula2022personalized}. The hybrid algorithm \hyb combines \cus and the traditional EM algorithm \cite{csimcsek2018expectation}, which provides a practical heuristic for mixing preference-consistent histories with ordinary independent transactions. These algorithms outperform the traditional EM algorithm for MC estimation  \cite{csimcsek2018expectation} and the benchmarks based on a partial-ordering-based heuristic for MNL estimation \cite{jagabathula2018partial} on synthetic datasets and the sushi dataset in a setting similar to \cite{berbeglia2022comparative}. The improvement is not only in customer-specific choice predictions but also in population choice predictions.

    \item As our secondary contribution, we present computational hardness and tractability results for conditional choice probability computation and conditional assortment optimization under MNL or MC for a specific customer. These results complement our estimation framework and clarify the computational landscape of conditional choice and assortment optimization under MNL and MC. In contrast to the previous work \cite{jagabathula2018partial,jagabathula2022personalized}, which considers partial ordering over \emph{all} products, our formulation considers partial ordering over \emph{products selected by the customer}. This formulation relies heavily on the structure of MNL or MC and could potentially reduce computational burden.
\end{itemize}

\subsection{Related Literature}

\subsubsection{EM Algorithms for Choice Model Estimation}

One of the most popular measures used for demand or choice model estimation is the likelihood metric. In some special cases, finding the optimal parameters is computationally tractable due to the concavity of the log-likelihood objective \cite{mcfadden1974conditional,van2015market}. However, the log-likelihood function is not well-structured in general. Standard optimization methods may become computationally intensive.

A common approach for addressing this computational burden is the EM algorithm \cite{dempster1977maximum}, widely used in demand censorship of substitutable products or choice model estimation \cite{vulcano2012estimating,van2017expectation,jagabathula2017nonparametric,anupindi1998estimation,talluri2004revenue,kok2007demand,conlon2013demand,stefanescu2009multivariate}. In the scope of choice model estimation, \cite{train2008algorithms} established EM algorithms for latent-class MNL and a discrete mixture of choice distributions. The latent-class MNL EM framework was later adapted to panel data settings in \cite{jagabathula2018partial,jagabathula2022personalized}. \cite{csimcsek2018expectation} presented the EM algorithm for MC, assuming that each transaction is independent. Our framework is an extension of \cite{csimcsek2018expectation}, which utilizes correlations among transactions from the same customer, thus improving the estimation of MC parameters.

\subsubsection{Choice Modeling and Assortment Planning with Panel Data}

Most assortment problems with panel data in the literature are based on the multinomial logit (MNL) model, a special case of MC. \cite{chen2023assortment} considered a multi-period assortment optimization problem under MNL, where the goal is to maximize the expected revenue over all time periods. \cite{jagabathula2018partial} initiated the study of partial-ordering-based customer preference estimation, later adapted to personalized promotion decision-making \cite{jagabathula2022personalized}. This framework estimated choice probabilities conditional on the customer's historical transactions under MNL and claimed that computing the exact probability is intractable. We present a formal proof of this statement and derive a closed-form expression to compute the exact choice probability under MC. In contrast to \cite{jagabathula2018partial,jagabathula2022personalized}, which utilize an efficient heuristic to estimate the choice probability under MNL, we focus on a more fine-grained computation under MC at a reasonable scale.

\subsubsection{Other Results for the MC Choice Model}

The MC choice model was pioneered in \cite{blanchet2016markov}, which presented the first polynomial-time algorithm for the assortment optimization problem under MC. Shortly after, \cite{feldman2017revenue} investigated the properties of the primal and dual linear programs (LP) for the assortment problem and an LP-based approximation for the network revenue management problem under MC. The capacitated assortment problem under MC is APX-hard \cite{desir2020constrained}, and there is a fully polynomial-time approximation scheme (FPTAS) when the rank of the MC transition matrix is constant \cite{desir2022capacitated}. \cite{desir2024robust} developed an efficient algorithm for the robust assortment problem under MC. The goal is to maximize the worst-case expected revenue when the MC parameters are uncertain. \cite{gupta2020parameter} presented an efficient algorithm for the identification of MC parameters under limited-size assortments and noiseless choice probabilities. \cite{li2025online} presented an online learning algorithm for the dynamic assortment selection problem under MC, which simultaneously involves online parameter estimation and sequential assortment decision-making.

\subsection{Organization}
In Section \ref{sec:pre}, we present the basic framework of choice modeling with panel data and introduce MC and MNL. In Section \ref{sec:val-cus}, we present the likelihood functions and the gap between MNL and MC regarding how exploiting individual transaction information could facilitate choice model estimation. In Section \ref{sec:em}, we present our EM algorithms for MC parameter estimation using panel data. In Section \ref{sec:ex}, we present our experimental results. We conclude in Section \ref{sec:con}. In Appendix \ref{sec:hnc}, we present hardness and computational results for conditional choice prediction and conditional assortment optimization.

\section{Preliminaries} \label{sec:pre}

In Section \ref{sec:cus}, we provide the background for random utility models (RUM) with panel data, including the basic settings, the connection between preference-based linear extensions and directed acyclic graphs (DAG), and the closed-form expression for conditional choice probability. In Section \ref{sec:scm}, we introduce specific choice models: the Markov chain (MC) and the multinomial logit (MNL). We end this section by presenting the closed-form expression for the probability of observing a specific set of transactions under MC and MNL.

In a RUM, there are $n$ substitutable products, denoted as the set $[n] = \{1, 2, ..., n\}$. The no-purchase option is denoted as product 0. A customer has a random utility $u_j \in \R$ for product $j$, following an underlying probability distribution. Depending on the RUM, the random utilities may be independent or not. Let $S \subseteq [n]$ be a subset of products offered to a customer. The no-purchase option is always available to the customer, so we use the notation $S_+ := S \cup \{0\}$. The customer purchases the product with the highest utility from $S_+$. That is, $j$ is chosen from $S_+$ if $j = \arg \max_{i \in S_+}\{u_i\}$.\footnote{We assume that the probability that multiple products have the highest utility is 0.} We use $\pi(j,S)$ to denote the probability that $j$ is chosen from $S_+$.

\subsection{Choice Modeling with Panel Data} \label{sec:cus}

Traditional studies have focused primarily on choice-probability computation, parameter estimation for a given class of RUM assuming that each transaction is independent, and assortment optimization when the RUM parameters are given. We propose variants of these problems in which each transaction is associated with a specific customer, potentially capturing correlations between transactions from the same customer.

Let $c$ denote a customer and $[C] = \{1, 2, ..., C\}$ denote the set of customers. Let $k_c$ denote the number of observed transactions associated with the customer $c$. A tuple $(j,S)$ captures a transaction in which the product $j \in S_+$ is purchased when $S$ is offered. Let $\cD^c:= \{(j^c_\ell, S^c_\ell) \mid \ell \in [k_c]\}$ be the set of observed transactions associated with the customer $c$. That is, the product $j^c_\ell$ is purchased when $S^c_\ell$ is offered in the $\ell$-th transaction of the customer $c$. In practice, $S^c_\ell$ can be the same for different labels $\ell$ and customers $c$, capturing the scenario in which different customers visit the same store during the same time period and face the same offer set. The notation $\cD^c$ will be used for different purposes in different sections. In Sections \ref{sec:val-cus}, \ref{sec:em}, and \ref{sec:ex}, $\cD^c$ is used as the collection of customer-specific transactions from different customers $c$ to estimate the parameters of MC. In Appendix \ref{sec:hnc}, we assume that the RUM parameters (more specifically, MC and MNL) are known and focus on the purchasing probability of a specific customer, conditional on the customer's transactions. In the remainder of this section, we focus on one specific customer by dropping the notation $c$.

Suppose the RUM and the historical transactions $\cD:= \{(j_\ell, S_\ell) \mid \ell \in [k]\}$ of a specific customer are given. We assume that the customer has underlying fixed utilities randomly generated according to the RUM. The customer has a strict ranking for the products in $[n]_+$. The seller only has access to $\cD$ as partial information. The transactions in $\cD$ must be consistent with a strict ranking. For example, having $S_1 = \{1,2\}$ and $j_1 = 1$ would imply $u_1 > \max\{u_0,u_2\}$, so if $S_2 = \{1\}$, we must have $j_2 = 1$. Without loss of generality, we assume that $\cD = \{(j_\ell, S_\ell) \mid \ell \in [k]\}$ with $j_\ell \ne j_{\ell'}$ for all $\ell \ne \ell'$. Otherwise, if $j_\ell = j_{\ell'}$, $(j_\ell, S_\ell)$ and $(j_{\ell'}, S_{\ell'})$ can be described as $(j_\ell, S_\ell \cup S_{\ell'})$ because $j_\ell$ is the most preferred product when $S_\ell \cup S_{\ell'}$ is offered.


We derive the probability that $j$ is chosen from $S$, conditional on the past transactions in $\cD$. Let $\cE(j,S)$ denote the event that $j$ is chosen from $S_+$. 
Let $\cE_\cD$ denote the event that $j_\ell$ is chosen from $S_\ell$ for all $\ell \in [k]$. Let $\pi_\cD$ be the probability that $\cE_\cD$ occurs. By definition, \begin{equation}\label{eq:def-cD}
    \pi_\cD = \Pr[\cE_\cD] = \Pr[\cap_{\ell \in [k]}\cE(j_\ell,S_\ell)].
\end{equation}
Let $\pi_\cD(j,S)$ denote the probability that $j$ is purchased when $S_+$ is offered, conditional on $\cE_\cD$.
By definition,
\begin{equation}\label{eq:cond-prob}
    \pi_\cD(j,S) = \frac{\Pr[\cE_\cD \cap \cE(j,S)]}{\pi_\cD}.
\end{equation}

We note that when the set of past transactions $\cD$ is empty, then $\pi_\cD = 1$ and $\pi_\cD(j,S) = \Pr[\cE(j,S)] = \pi(j,S)$. This captures the unconditional choice probability.

To compute the conditional probability $\pi_\cD(j,S)$ in \eqref{eq:cond-prob}, it suffices to calculate the probability of observing a transaction set separately for the numerator and denominator. The event $\cE_\cD \cap \cE(j,S)$ is equivalent to $\cE_{\cD'}$ where $\cD' = \cD \cup \{(j,S)\}$. That is, we add transaction $(j,S)$ and compute the probability in the numerator.

We proceed to introduce the directed acyclic graph (DAG) based presentation that is equivalent to the transaction-based presentation, the notion of complete linear extensions that captures all the possible strict rankings over products, and the conditional assortment optimization problems.


\subsubsection{A DAG-based Presentation for Customer Information} \label{subsec:dag}

An alternative presentation for partial customer information is a DAG \cite{jagabathula2018partial,jagabathula2022personalized}. A DAG $D=([n]_+,A)$ has products $j \in [n]_+$ representing the vertices and a set of arcs $A \subseteq [n]_+ \times [n]_+$ where $(j, j') \in A$ denotes that $j$ is preferred over $j'$, that is, $u_j > u_{j'}$. We note that transitive arcs can be removed for the sake of presentation. For example, suppose $(j, j')$, $(j', j'')$, and $(j, j'')$ are in $A$, then we can remove $(j, j'')$ because the other two arcs imply that $u_j > u_{j'} > u_{j''}$.

The DAG-based presentation is equivalent to the transaction-based presentation. Given a set of historical transactions $\cD = \{(j_\ell, S_\ell) \mid \ell \in [k]\}$ consistent with an underlying strict ranking, an equivalent DAG can be constructed as follows. For each $\ell \in [k]$ and $j \in {S_\ell}_+ \setminus \{j_\ell\}$, add $(j_\ell, j)$ to $A$. That is, for each transaction, connect the purchased product to each product that was not purchased to capture the customer's pairwise preference. On the other hand, given a DAG that captures the customer's preference, a set of historical transactions can be constructed as follows. With a specific ordering, let $j_\ell$ denote the $\ell$-th product that has at least one outgoing arc. For each $j_\ell$, define the offer set as $S_\ell =(\{j_\ell\} \cup \{j' \mid (j_\ell,j') \in A\}) \setminus \{0\}$.


We note that in empirical applications, customers might purchase and try different products, and update their idiosyncratic preferences from time to time. Considering the transactions over a time period may introduce cyclic preferences. In this scenario, cycle-removal heuristics can be used to find a DAG that best describes the customer's preference \cite{jagabathula2018partial,jagabathula2022personalized}.

\subsubsection{Complete Linear Extensions} \label{sec:cle}

Given a set of transactions $\cD = \{(j_\ell, S_\ell) \mid \ell \in [k]\}$ (or a DAG $D = ([n]_+, A)$), a \emph{complete linear extension} $\sigma: [n+1] \to [n]_+$ is a strict preference ranking over $[n]_+$ that is consistent with $\cD$ (or $D$).\footnote{A linear extension is also called a topological ordering.} Here, $\sigma(s)$ is the $s$-th most preferred product among $[n]_+$. For clarity, sometimes we use $(\sigma(1), \sigma(2), ..., \sigma(n+1))$ to denote $\sigma$. We use $\sigma \sim_c \cD$ to denote that the complete linear extension $\sigma$ is consistent with $\cD$.

Let $\pi_\sigma$ be the probability that the customer's strict preference ranking over $[n]_+$ is $\sigma$. Then the probability that $\cE_\cD$ occurs is
\begin{equation} \label{eq:prob-cle}
    \pi_\cD = \sum_{\sigma \sim_c \cD} \pi_\sigma.
\end{equation}

Using \eqref{eq:prob-cle} could be computationally intensive. 
In Section \ref{sec:rle}, we present a simplified form for \eqref{eq:prob-cle} under MC and MNL, to potentially reduce computation burden.

\begin{example} \label{ex:dag}
    Let $n=4$ and $\cD = \{(j_1, S_1), (j_2, S_2), (j_3, S_3)\}$ with $(j_1, S_1) = (1, \{1, 3\})$, $(j_2, S_2) = (2, \{2, 3, 4\})$, and $(j_3, S_3) = (3, \{3, 4\})$. The DAG for $\cD$ is presented in Figure \ref{fig:dag}. 
    
    Products $1$ and $2$ have higher utilities than product $3$ and no-purchase. Product $3$ has a higher utility than product $4$ and no-purchase. We do not know whether product 1 is preferred to product 2, and whether product 4 is preferred to no-purchase. The possible complete linear extensions are $(1,2,3,4,0)$, $(1,2,3,0,4)$, $(2,1,3,4,0)$, and $(2,1,3,0,4)$, representing the possible rankings for all products.
    Suppose that the RUM is such that the complete linear extensions $(1,2,3,4,0)$, $(1,2,3,0,4)$, $(2,1,3,4,0)$, and $(2,1,3,0,4)$ occur with probability $0.2$, $0.1$, $0.2$, and $0.1$, respectively. Then $\pi_\cD = 0.6$ by \eqref{eq:prob-cle}.

    \begin{figure}[h]
    \centering
    \begin{tikzpicture}
        \node[vertex] (0) at (7,2) {$0$};
        \node[vertex] (1) at (0,2.3) {$1$};
        \node[vertex] (2) at (0,1) {$2$};
        \node[vertex] (3) at (2,1.7) {$3$};
        \node[vertex] (4) at (7,1) {$4$};
        \path[->] (1) edge (0)
        (1) edge (3)
        (2) edge (0)
        (2) edge (3)
        (2) edge (4)
        (3) edge (0)
        (3) edge (4)
        ;
    \end{tikzpicture}
    \caption{The DAG for $\cD$}
    \label{fig:dag}
\end{figure}
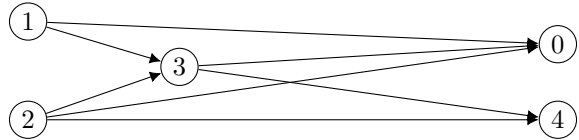 
\end{example}

\subsubsection{Conditional Assortment Optimization}

In assortment problems, each product $j \in [n]_+$ is associated with a revenue $r_j \in \R_{\ge 0}$. In particular, the no-purchase option has zero revenue, so $r_0 = 0$. We use $r:=\{r_j\}_{j \in [n]_+}$ to denote the revenue vector. The traditional assortment optimization (TAO) problem aims to find an offer set $S \subseteq [n]$ that maximizes the expected revenue:

\begin{equation} \label{opt:tao}
\max_{S \subseteq [n]} \sum_{j \in S} \pi(j,S) r_j. \tag{TAO}
\end{equation}

In the conditional assortment optimization (CAO) problem, the seller aims to maximize the expected revenue for a specific customer, conditional on the transactions $\cD$. That is,
\begin{equation} \label{opt:cao}
    \max_{S \subseteq [n]} \sum_{j \in S} \pi_\cD(j,S) r_j. \tag{CAO}
\end{equation}
In the \emph{transparent} setting where the seller cannot hide the products presented in the past, it is required to offer a set that includes all products that were purchased. We use $S \sim_t \cD$ to denote that $S \subseteq [n]$ satisfies $j_\ell \in S_+$ for all $\ell \in [k]$. The transparent conditional assortment optimization (TCAO) problem is
\begin{equation} \label{opt:tcao}
    \max_{S \sim_t \cD} \sum_{j \in S} \pi_\cD(j,S) r_j. \tag{TCAO}
\end{equation}

\subsection{Specific Choice Models} \label{sec:scm}

In this section, we introduce the Markov chain (MC) choice model and its special case, the multinomial logit (MNL) choice model. Both models are special cases of RUMs \cite{berbeglia2016discrete}. We derive simpler closed-form probabilities in equations \eqref{eq:def-cD}, \eqref{eq:cond-prob}, and \eqref{eq:prob-cle} under these two models.

\subsubsection{Markov Chain Choice Model} \label{subsubsec:mc}

Throughout the paper, we primarily consider the Markov chain (MC) choice model. Under MC, each product in $[n]_+$ is associated with a state. Each state $i \in [n]_+$ is associated with a given parameter $\lambda_i$ which denotes the probability of starting at state $i$. Each state $i \in [n]$ has a transition probability to state $j \in [n]_+$, denoted as $\rho_{ij}$. Naturally, we have $\sum_{i=0}^n \lambda_i = 1$ and $\sum_{j=0}^n \rho_{ij} = 1$ for all $i \in [n]$. We use $\lambda := \{\lambda_j\}_{j\in[n]_+}$ to denote the initial probability vector and $\rho:=\{\rho_{ij}\}_{i \in [n], j\in [n]_+}$ to denote the transition matrix. The customer enters the system according to the initial probability $\lambda$ and performs a Markovian random walk according to $\rho$. Given an offer set $S \subseteq [n]$ and suppose the customer is at state $i$. If $i \in S_+$, then the customer selects product $i$ and leaves the system. Otherwise, if $i \notin S_+$, the customer transitions to state $j$ with probability $\rho_{ij}$. The customer performs this Markovian random walk until it reaches a state in $S_+$.

Recall that $\pi(j,S)$ denotes the probability that the product $j$ is chosen from $S_+$. Let $\overline{S}:=[n] \setminus S$ denote the complement of the offer set $S$. In \cite{blanchet2016markov,feldman2017revenue}, $\pi(j,S)$ is derived from the following system of linear equations:
\begin{equation} \label{lineq:mc}
    \pi(j,S) = \lambda_j + \sum_{i \in \overline{S}} \rho_{ij} \pi(i,S) \quad \forall j \in [n]_+.
\end{equation}
On the left-hand side, $\pi(j,S)$ is the expected number of times that a customer visits state $j$ during the course of the choice process when $S$ is offered. When a customer visits state $j$, it either enters the system directly at state $j$ or transitions from a state $\overline{S}$ to state $j$. The former case yields the term $\lambda_j$ on the right-hand side. In the latter case, the expected number of times that a customer visits a state $i \in \overline{S}$ is $\pi(i,S)$. In each of these visits, the customer transitions from state $i$ to state $j$ with probability $\rho_{ij}$, yielding the term $\sum_{i \in \overline{S}} \rho_{ij} \pi(i,S)$ on the right-hand side. When $j \in S$, since $j$ can be visited at most once, the expected number of visits to $j$ is the same as the probability that $j$ is chosen from $S$. 

From \eqref{lineq:mc}, we proceed to derive the closed-form expression for $\pi(j,S)$. For any $T \subseteq [n]_+$, let $\pi(T,S):=\{\pi(j,S)\}_{j \in T}$ denote the vector that captures the expected number of times state $j \in T$ is visited when $S$ is offered. Let $\lambda(T):=\{\lambda_j\}_{j \in T}$ denote a sub-vector of $\lambda$ with coordinates in $T$. For any $T' \subseteq [n]$ and $T \subseteq [n]_+$, let $\rho(T',T):=\{\rho_{ij}\}_{i \in T', j \in T}$ denote the sub-matrix of $\rho$ that captures a transition from states in $T'$ to states in $T$. From \eqref{lineq:mc}, we observe that $\pi(\overline{S},S) = \lambda(\overline{S}) + \rho(\overline{S}, \overline{S})^T \pi(\overline{S},S)$, which implies that
\begin{equation} \label{eq:pi-bar-S}
    \pi(\overline{S},S) = \left((I-\rho(\overline{S}, \overline{S}))^T\right)^{-1} \lambda(\overline{S}) = \left((I-\rho(\overline{S}, \overline{S}))^{-1}\right)^T \lambda(\overline{S}).
\end{equation}
To ensure that $(I-\rho(\overline{S}, \overline{S}))^{-1}$ is unique and always exists, we assume that $\sum_{i \in \overline{S}} \rho_{ij} < 1$ for all $S \subseteq [n]$ \cite{puterman2014markov}.\footnote{This assumption is equivalent to having the spectral radius of $\rho([n],[n])$ strictly less than one.} From \eqref{lineq:mc} and \eqref{eq:pi-bar-S}, we have that 
\begin{align}
    \pi(S_+,S) &= \lambda(S_+) + \rho(\overline{S}, S_+)^T \pi(\overline{S},S) \nonumber\\
    &= \lambda(S_+) + \rho(\overline{S}, S_+)^T \left((I-\rho(\overline{S}, \overline{S}))^{-1}\right)^T \lambda(\overline{S}) \label{eq:pi-S}.
\end{align}
We interpret the right-hand side of \eqref{eq:pi-S} as follows. When a customer enters the system, it either directly enters a state in $S_+$, or starts from a state in $\overline{S}$, and transitions $q$ additional times in $\overline{S}$ for $q \ge 0$ before reaching a state in $S_+$. The former case is captured by the probability vector $\lambda(S_+)$ while the latter case is captured by $\rho(\overline{S}, S_+)^T \left(\sum_{q=0}^\infty \rho(\overline{S}, \overline{S})^q\right)^T \lambda(\overline{S}) = \rho(\overline{S}, S_+)^T \left((I-\rho(\overline{S}, \overline{S}))^{-1}\right)^T \lambda(\overline{S})$.

Under MC, the following linear program (LP) introduced in \cite{feldman2017revenue} solves \eqref{opt:tao} in polynomial time.
\begin{equation} \label{lp:mc-tao}
    \min_g \sum_{i \in [n]} \lambda_i g_i \quad \text{s.t.} \quad g_i \ge r_i \text{ and } g_i \ge \sum_{j \in [n]} \rho_{ij} g_j \quad \forall i \in [n].
\end{equation}
Here, the variable $g_i$ denotes the maximum expected revenue obtained while visiting state $i$. The objective $\min_g \sum_{i \in [n]} \lambda_i g_i$ represents the expected revenue and forces $g_i = \max\{r_i, \sum_{j \in [n]} \rho_{ij} g_j\}$. If $g_i = r_i$, then offering product $i$ maximizes the expected revenue while visiting state $i$. Otherwise, if $g_i > r_i$, then not offering product $i$ forces the customer to transition to another state $j$, and $g_i = \sum_{j \in [n]} \rho_{ij} g_j$ maximizes the expected revenue while visiting state $i$. Let $g^*$ be the optimal solution to \eqref{lp:mc-tao}, then offering $\{i \mid g^*_i = r_i\}$ maximizes the expected revenue.

\subsubsection{Multinomial Logit Choice Model}
The Multinomial Logit (MNL) model is by far the most popular in practice \cite{bradley1952rank,luce1959individual,mcfadden1974conditional,plackett1975analysis}.
Under MNL, the utility of product $j$ takes the form $u_j = \mu_j + \varepsilon_j$ where $\mu_j$ is the mean of $u_j$ that is fixed and $\varepsilon_j$ is a standard Gumbel random variable. The Gumbel shocks $\ep_j$
 are independent, so the utilities $u_j$ are independent. Let $v_j:=\exp(\mu_j) > 0$ be the attraction value of product $j$. Then the choice probability 
\begin{equation} \label{eq:mnl-prob}
\pi(j,S) = \frac{v_j}{\sum_{i \in S_+} v_i} = \frac{v_j}{V(S_+)}
\end{equation}
where $V(T):=\sum_{j \in T}v_j$ for ease of notation.
MNL can be represented by MC, where the transition matrix $\rho$ is rank one. More specifically, an MNL with parameter vector $v:=\{v_j\}_{j \in [n]_+}$ can be captured by an MC with parameters $\lambda_j = v_j / V([n]_+)$ for all $j \in [n]_+$ and $\rho_{ij} = v_j / V([n]_+)$ for all $i \in [n], j \in [n]_+$ \cite{blanchet2016markov}.

Under MNL, \eqref{opt:tao} can be solved in polynomial time by considering revenue-ordered assortments \cite{gallego2004managing,talluri2004revenue}. An assortment $S \subseteq [n]$ is revenue-ordered if $S$ includes the top $s$ revenue products for some $s \in [n]$.

\subsubsection{Closed-form for the Probability of a Set of Transactions under MC or MNL} \label{sec:rle}
Given a set of transactions $\cD = \{(j_\ell, S_\ell) \mid \ell \in [k]\}$, we derive a closed-form expression for $\pi_\cD$ under MC or MNL. Section \ref{sec:cle} equation \eqref{eq:prob-cle} takes the summation of the probabilities of having each complete linear extension. However, this could cause computational burden. Under MC or MNL, an alternative is to consider \emph{reduced linear extensions} represented as rankings over the subscript indices of the chosen products $\{j_\ell \mid \ell \in [k]\}$. Recall that the customer has an underlying strict preference, and we assume that $j_\ell \ne j_{\ell'}$ for all $\ell \ne \ell'$. For any $\ell \ne \ell'$ where $(j_\ell, S_\ell), (j_{\ell'}, S_{\ell'}) \in \cD$, if $j_{\ell'} \in S_{\ell}$, then $j_\ell$ is more preferred than $j_{\ell'}$, so we have $\ell \succ \ell'$. Considering all pairs $\ell,\ell'$ and transitivity, this defines the partial ordering for $[k]$ based on $\cD$.
A reduced linear extension $\sigma: [k] \to [k]$ is a strict preference ranking over $[k]$ that is consistent with $\cD$. Here, $j_{\sigma(s)}$ is the $s$-th most preferred product among $\{j_\ell \mid \ell \in [k]\}$. We note that if $j_{\ell'} = 0$, then $\sigma(k) = \ell'$, i.e., the no-purchase option is the least preferred among $\{j_\ell \mid \ell \in [k]\}$. Since $0 \in {S_\ell}_+$ for all $\ell \in [k]$, we do not know the ranking over products less preferred than the no-purchase option. For clarity, sometimes we use $(\sigma(1), \sigma(2), ..., \sigma(k))$ to denote $\sigma$. We use $\sigma \sim_r \cD$ to denote that the reduced linear extension $\sigma$ is consistent with $\cD$. 

The partial ordering $\succ$ for $[k]$ can be presented as a \emph{reduced DAG}, which has $[k]$ as the set of vertices and arcs that represent the partial ordering $\succ$. That is, $\ell \succ \ell'$ if and only if there is an $\ell \leadsto \ell'$ path in the reduced DAG, capturing the case that $j_\ell$ is more preferred than $j_{\ell'}$.

\paragraph{\bf Markov Chain.} Under MC, the ranking for the chosen products $\{j_\ell \mid \ell \in [k]\}$ is decided by the customer's random walk, i.e., the course of product selection. The ordering for the first appearance of each product in the random walk defines the customer's preference. In other words, if the first visit of state $i$ appears before the first visit of state $j$ in the random walk, then product $i$ is more preferred than product $j$.

Following this observation, any random walk that is consistent with $\cD$ must be in the form of a reduced linear extension $\sigma \sim_r \cD$. Suppose that the random walk is in the form of $\sigma$. Let $U:=\cup_{\ell \in [k]} S_\ell$. Before the first visit of $j_{\sigma(1)}$, it is impossible that the random walk has visited a state in $U \setminus \{j_{\sigma(1)}\}$. Otherwise, there exists a state $j \in \{j_\ell \mid \ell \in [k]\} \setminus \{j_{\sigma(1)}\}$ that is more preferred than $j_{\sigma(1)}$, violating the form of $\sigma$. A random walk that is in the form of $\sigma$, up to the first visit of $j_{\sigma(1)}$, has the following prefix form: the customer either directly enters $j_{\sigma(1)}$ or enters a state in $\overline{U}$, traverses states in $\overline{U}$, and transitions from a state in $\overline{U}$ to $j_{\sigma(1)}$. The probability of having this prefix random walk is $\pi(j_{\sigma(1)}, U)$, the probability that $j_{\sigma(1)}$ is chosen from $U$.

The random walk then proceeds from the first visit of $j_{\sigma(1)}$ to the first visit of $j_{\sigma(2)}$, through states in $\overline{\cup_{\ell=2}^{k} S_{\sigma(\ell)}} = \cap_{\ell=2}^{k} \overline{S_{\sigma(\ell)}}$. For convenience, let $U^\sigma_s:=\cup_{\ell=s}^{k} S_{\sigma(\ell)}$ be the \emph{union of the absorbing sets} with subscript indices from $\sigma(s)$ to $\sigma(k)$ and $A^\sigma_s:=\overline{U^\sigma_s}=\overline{\cup_{\ell=s}^{k} S_{\sigma(\ell)}} = \cap_{\ell=s}^{k} \overline{S_{\sigma(\ell)}}$ be the \emph{admissible set} from the first visit of $\sigma(s-1)$ to the first visit of $\sigma(s)$ for $s \in \{2, 3, ..., k\}$. We denote the event of having a random walk from state $i$ to state $j$ through states in $T$ as $i \overset{T}{\leadsto} j$. Following the same reasoning for \eqref{eq:pi-S}, we have that
\begin{equation} \label{prob:i-to-j}
    \Pr[i {\overset{T}{\leadsto}} j]:=\rho_{ij} + \rho(T,\{j\})^T \left(\left(I-\rho(T,T)\right)^{-1}\right)^T\rho(\{i\},T)^T 
    = \rho_{ij} + \rho(\{i\},T) \left(I-\rho(T,T)\right)^{-1}\rho(T,\{j\}).
\end{equation}
The random walk from the first visit of $j_{\sigma(1)}$ to the first visit of $j_{\sigma(2)}$, through states in $A^\sigma_2$, occurs with probability $\Pr[j_{\sigma(1)} {\overset{A^\sigma_2}{\leadsto}} j_{\sigma(2)}]$ by \eqref{prob:i-to-j}. Using analogous arguments, the random walk from the first visit of $j_{\sigma(s)}$ to the first visit of $j_{\sigma(s+1)}$ for each $s \in [k-1]$, through states in $A^\sigma_{s+1}$, occurs with probability $\Pr[j_{\sigma(s)} {\overset{A^\sigma_{s+1}}{\leadsto}} j_{\sigma(s+1)}]$. When the customer's random walk is in the form of $\sigma$, the likelihood factors over the prefix segment and the subsequent inter-arrival segments. We have
\begin{equation} \label{eq:prob-rle-mc}
    \pi^{\cD}_\sigma = \pi(j_{\sigma(1)}, U)\prod_{s = 1}^{k-1}\Pr[j_{\sigma(s)} {\overset{A^\sigma_{s+1}}{\leadsto}} j_{\sigma(s+1)}]
\end{equation}
where $\pi^{\cD}_\sigma$ is the probability of having a random walk consistent with the reduced linear extension $\sigma$ given the transaction set $\cD$. The probability that $\cE_\cD$ occurs is
\begin{equation} \label{eq:prob-rle}
    \pi_\cD = \sum_{\sigma \sim_r \cD} \pi^\cD_\sigma.
\end{equation}
We note that \eqref{eq:prob-cle} accounts for the probability $\pi_\sigma$ of having a random walk consistent with a complete linear extension $\sigma$, which depends on the RUM and is independent of $\cD$. In \eqref{eq:prob-rle}, the probability $\pi^\cD_\sigma$ of having a random walk consistent with a reduced linear extension $\sigma$ depends on $\cD$.

\paragraph{\bf Multinomial Logit.}

Under MNL, when the customer preference is consistent with a reduced linear extension $\sigma \sim_r \cD$, $\sigma$ must be in the following form. First, products ranked before $j_{\sigma(1)}$ must belong to $\overline{U}$, so $j_{\sigma(1)}$ is the most preferred among $U_+$. After $j_{\sigma(1)}$, products ranked before $j_{\sigma(2)}$ must belong to $A^\sigma_2$, so $j_{\sigma(2)}$ is the most preferred among ${U^\sigma_2}_+$. Repeating the same argument, for any $s \in [k]$, $j_{\sigma(s)}$ is the most preferred among ${U^\sigma_s}_+$. According to \cite{beggs1981assessing},
\begin{equation} \label{eq:prob-rle-mnl}
    \pi^{\cD}_\sigma = \prod_{s=1}^k\pi(j_{\sigma(s)}, {U^\sigma_s}) = \prod_{s=1}^k \frac{v_{j_{\sigma(s)}}}{V\left({{U^\sigma_s}}_+\right)}.
\end{equation}

\begin{example} \label{ex:dag-rle}
    Recall that in Example \ref{ex:dag}, $n=4$ and $\cD = \{(j_1, S_1), (j_2, S_2), (j_3, S_3)\}$ with $(j_1, S_1) = (1, \{1, 3\})$, $(j_2, S_2) = (2, \{2, 3, 4\})$, and $(j_3, S_3) = (3, \{3, 4\})$. The partial ordering is defined by $1 \succ 3$ and $2 \succ 3$, so both $j_1$ and $j_2$ are more preferred than $j_3$. The possible reduced linear extensions $\sigma \sim_r \cD$ are $(1,2,3)$ and $(2,1,3)$. The reduced DAG is presented in Figure \ref{fig:r-dag}.
    Under MC, following \eqref{eq:prob-rle-mc}, we have that
    \begin{align*}
        \pi^{\cD}_{(1,2,3)} = \pi(1, [4]) \Pr[1 {\overset{\{1\}}{\leadsto}} 2] \Pr[2 {\overset{\{1,2\}}{\leadsto}} 3]
        \text{ and } \pi^{\cD}_{(2,1,3)} = \pi(2, [4]) \Pr[2 {\overset{\{2\}}{\leadsto}} 1] \Pr[1 {\overset{\{1,2\}}{\leadsto}} 3].
    \end{align*}

    Under MNL, following \eqref{eq:prob-rle-mnl}, we have that
    \begin{align*}
        \pi^{\cD}_{(1,2,3)} = \pi(1, [4]) \pi(2, \{2,3,4\}) \pi(3,\{3,4\})
        \text{ and } \pi^{\cD}_{(2,1,3)} = \pi(2, [4]) \pi(1, \{1,3,4\}) \pi(3,\{3,4\}).
    \end{align*}
    \begin{figure}[h]
    \centering
    \begin{tikzpicture}
        \node[vertex] (1) at (0,0) {$1$};
        \node[vertex] (2) at (6,0) {$2$};
        \node[vertex] (3) at (3,0) {$3$};
        \path[->]
        (1) edge (3)
        (2) edge (3)
        ;
    \end{tikzpicture}
    \caption{The reduced DAG for $\cD$}
    \label{fig:r-dag}
\end{figure}
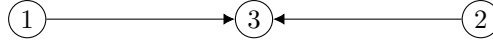 
\end{example}

\section{The Value of Panel Information} \label{sec:val-cus}

In this section, we investigate the value of using panel data while estimating a choice model. In Section \ref{sec:llh}, we present the log-likelihood functions, which measure the performance of choice model parameter estimation. In Section \ref{sec:gap}, we provide an example that highlights how exploiting individual transaction information can improve MC parameter estimation. In Section \ref{sec:pp}, we show that the MNL parameters can be identified without linking the transactions to specific customers when the choice probability vectors of a few assortments are given. Sections \ref{sec:gap} and \ref{sec:pp} highlight the gap between MNL and MC in how exploiting individual transaction information can facilitate choice model estimation.

\subsection{The Log-likelihood Objectives} \label{sec:llh}

Suppose we have a set of observable transactions. Each transaction is associated with a customer identity (ID), the offer set, and the product chosen from the offer set. In the traditional setting, the customer ID is not revealed, and transactions are treated as independent observational units. In the setting with panel data, the customer ID of each transaction is revealed. Customers are treated as independent observational units, while transactions from the same customer jointly induce a partial order.

More specifically, in the traditional setting, we are given a set of transactions $\cI = \{(j_r, S_r) \mid r \in [I]\}$ without the customer IDs. Here, $\cI$ has $I$ transactions $(j_r, S_r)$ denoting that $j_r$ is chosen from the offer set $S_r$ in transaction $r \in [I]$. Each transaction in $\cI$ is assumed to be independent. The likelihood of having $\cI$ is $\prod_{r=1}^I \pi(j_r,S_r)$, and the log-likelihood function is
\begin{equation} \label{eq:log-like-trad}
    L^{trad}(.) = \sum_{r=1}^I \log \pi(j_r,S_r)
\end{equation}
where $L^{trad}$ and $\pi(j_r,S_r)$ depend on the choice model parameters to be selected. A choice model estimation algorithm aims to maximize the log-likelihood objective \eqref{eq:log-like-trad} by selecting the best parameters. The log-likelihood objective \eqref{eq:log-like-trad} for MNL is concave. For MC, the structure of \eqref{eq:log-like-trad} is more complicated and thus requires a more sophisticated algorithm like EM \cite{csimcsek2018expectation} to find the parameters that achieve a local maximum.

In the setting with panel data, the customer ID of each transaction is revealed, so each customer is associated with a set of transactions. Recall that $[C]$ is the set of customers, and each customer $c \in [C]$ has a set of transactions $\cD^c:=\{(j^c_\ell, S^c_\ell) \mid \ell \in [k_c]\}$ where $k_c$ is the number of transactions in $\cD^c$. Without loss of generality, we assume that $j^c_\ell \ne j^c_{\ell'}$ for all $\ell \ne \ell'$. Each set of transactions $\cD^c$ is consistent with at least one strict preference list for $[n]_+$, that is, there are no cycles induced by $\cD^c$, and a reduced DAG for $\cD^c$ can be constructed. It is assumed that the underlying strict preference of each customer is independently drawn. Recall that $\cE_{\cD^c}$ is the event that customer $c$ selects $j^c_\ell$ from ${S^c_\ell}_+$ for all $\ell \in [k_c]$ and $\pi_{\cD^c}:=\Pr[\cE_{\cD^c}]$. The likelihood of having $\{\cD^c\}_{c \in [C]}$ is $\prod_{c=1}^C \pi_{\cD^c}$, and the log-likelihood function is
\begin{equation} \label{eq:log-like-cus}
    L^{cus}(.) = \sum_{c=1}^C \log \pi_{\cD^c}
\end{equation}
where $L^{cus}$ and $\pi_{\cD^c}$ depend on the choice model parameters to be selected. A choice model estimation algorithm aims to maximize the log-likelihood objective \eqref{eq:log-like-cus} by selecting the best parameters. Unlike the traditional setting, for both MNL and MC, the objective \eqref{eq:log-like-cus} is not concave, and computing $\pi_{\cD^c}$ is \#P-hard given the choice model parameters, as shown in Theorem \ref{thm:sp} in the appendix. For computational efficiency, an approximation for \eqref{eq:log-like-cus} under MNL, which is a concave function, is used as a heuristic in \cite{jagabathula2018partial}.

\subsection{Improving MC Parameter Estimation: An Example} \label{sec:gap}

In this section, we present an example showing that using panel data for choice model estimation could be beneficial for MC.

Parameter identification in MC was studied in \cite{gupta2020parameter}. For any $r \in \{2,3, ..., n-1\}$, when the choice probability vectors $\pi(S_+,S)$ for all $S$ of size $r$ or $r+1$ are given as input, the MC parameters can be perfectly identified. We present an example where the set of assortments is limited, so that parameter identification becomes challenging in the traditional setting, but linking transactions from the same customer could provide extra information to facilitate MC parameter estimation. 

Let the parameter of the ground truth MC (expressed with superscript $gt$) with $n=2$ be as follows: $\lambda^{gt}_1=0.6$, $\lambda^{gt}_2 = 0.4$, $\rho^{gt}_{10}=1$, $\rho^{gt}_{20}=1$, and all other parameters are zero. The goal is to estimate the MC parameters $(\lambda, \rho)$ that maximize \eqref{eq:log-like-trad} or \eqref{eq:log-like-cus}. Without loss of generality, we assume that there are no self-cycles in the choice process, i.e., $\rho_{11} = \rho_{22} = 0$. Therefore, it is sufficient to estimate four parameters: $\lambda_1, \lambda_2, \rho_{12}$, and $\rho_{21}$, so that $\lambda_0 = 1-\lambda_1-\lambda_2$, $\rho_{10} = 1-\rho_{12}$, and $\rho_{20} = 1 - \rho_{21}$. The illustration for the ground truth MC and the MC parameters to be determined is presented in Figure \ref{fig:mc}.

    \begin{figure}[h]
    \centering
    \begin{tikzpicture}
        \node[vertex] (1) at (0,0) {$1$};
        \node[vertex] (2) at (6,0) {$2$};
        \node[vertex] (3) at (3,0) {$0$};
        \path[->]
        (1) edge node [below] {$\rho^{gt}_{10}$=1} (3)
        (2) edge node [below] {$\rho^{gt}_{20}$=1}  (3)
        ;
        \node at (0,0.5) {$\lambda^{gt}_1 = 0.6$};
        \node at (6,0.5) {$\lambda^{gt}_2 = 0.4$};
        \node at (10,0.5) {$\lambda = [1-\lambda_1-\lambda_2, \lambda_1, \lambda_2]^T$};
        \node at (10,-0.3) {$\rho = \begin{bmatrix}
            1-\rho_{12} & 0 & \rho_{12} \\
            1-\rho_{21} & \rho_{21} & 0
        \end{bmatrix}$};
    \end{tikzpicture}
    \caption{The ground truth MC and the parameters to be determined.}
    \label{fig:mc}
\end{figure}
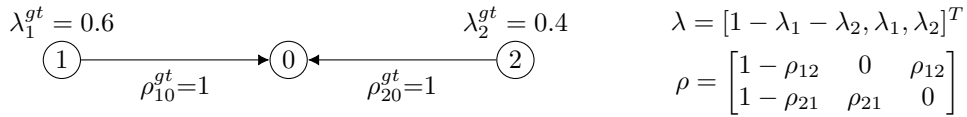 

From Figure \ref{fig:mc}, there are two types of customers with the following strict preferences: $1 \succ 0 \succ 2$ and $2 \succ 0 \succ 1$, which occur with probability 0.6 and 0.4, respectively. Each customer, randomly and independently drawn from the ground truth MC, is offered $\{1\}$ and $\{2\}$ once each. With $C$ randomly drawn customers, we have $I=2C$ transactions. The customer choice (selected product) from the combination of each customer type and each assortment is presented in Table \ref{table:choice}.

    \begin{table}[!htb]
\begin{center}
    \begin{tabular}{|c|c|c|}
        \hline
        customer type & $\{1\}$ & $\{2\}$ \\
        \hline
        $1 \succ 0 \succ 2$ (60\%) & 1 & 0 \\
        \hline
        $2 \succ 0 \succ 1$ (40\%)& 0 & 2 \\
        \hline
    \end{tabular}
\end{center}
\caption{Customer choice of each combination.}
\label{table:choice}
\end{table}
    
    In the traditional setting, when the number of customers $C$ becomes large, the observed choice probabilities converge to $\tilde\pi(0,\{1\})=0.4$, $\tilde\pi(1,\{1\})=0.6$, $\tilde\pi(0,\{2\})=0.6$, and $\tilde\pi(2,\{2\})=0.4$. Here, $\tilde\pi(j,S)$ is the choice probability that $j$ is chosen from $S_+$. Different from $\pi(j,S)$, which is the estimated choice probability that depends on the MC parameters, $\tilde\pi(j,S)$ is the given input. From \eqref{eq:log-like-trad}, we aim to maximize
    \begin{align}
        & \tilde\pi(1,\{1\}) \log \pi(1,\{1\}) + \tilde\pi(0,\{2\}) \log \pi(0,\{2\}) + \tilde\pi(0,\{1\}) \log \pi(0,\{1\}) + \tilde\pi(2,\{2\}) \log \pi(2,\{2\}) \nonumber \\
        = \quad & 0.6 \log(\lambda_1 + \lambda_2 \rho_{21}) + 0.6 \log(\lambda_0 + \lambda_1 \rho_{10}) + 0.4 \log(\lambda_0 + \lambda_2 \rho_{20}) + 0.4 \log(\lambda_2 + \lambda_1 \rho_{12}) \nonumber \\
        = \quad & 0.6 \log(\lambda_1 + \lambda_2 \rho_{21}) + 0.6 \log((1-\lambda_1-\lambda_2) + \lambda_1 (1-\rho_{12})) \nonumber \\
        & + 0.4 \log((1-\lambda_1-\lambda_2) + \lambda_2 (1-\rho_{21})) + 0.4 \log(\lambda_2 + \lambda_1 \rho_{12}) \label{eq:ex-trad-llh}    \end{align}
        subject to $\lambda_1+\lambda_2 \le 1$, and $\lambda_1, \lambda_2, \rho_{12}, \rho_{21} \in [0,1]$. Here, the factor $C$ is omitted, and we focus on the proportion of the choice outcomes. We have that $\pi(1,\{1\}) = \lambda_1 + \lambda_2 \rho_{21}$. When $\rho_{11} = \rho_{22} = 0$, product 1 is chosen from $\{0,1\}$ because either the customer starts from state 1 and ends the choice process or starts from state 2, transitions to state 1, then ends the choice process. The reasoning for $\pi(0,\{1\})$, $\pi(0,\{2\})$, and $\pi(2,\{2\})$ follows analogously.

        One can verify that \eqref{eq:ex-trad-llh} is maximized when
    \begin{align*}
        \pi(1,\{1\}) &= \lambda_1 + \lambda_2 \rho_{21} = 0.6, \\
        \pi(0,\{1\}) &= (1-\lambda_1-\lambda_2) + \lambda_2 (1-\rho_{21}) = 0.4, \\
        \pi(2,\{2\}) &= \lambda_2 + \lambda_1 \rho_{12} = 0.4, \text{and} \\
        \pi(0,\{2\}) &= (1-\lambda_1-\lambda_2) + \lambda_1 (1-\rho_{12}) = 0.6.
    \end{align*}
    The optimal solutions are parameterized by $\lambda_1$ and $\lambda_2$:
    \begin{align}
        \rho_{12} = \frac{0.4-\lambda_2}{\lambda_1} < 1 \text{ and }
        \rho_{21} = \frac{0.6-\lambda_1}{\lambda_2} < 1,
        \text{where } \lambda_1 \in (0,0.6],
        \lambda_2 \in (0,0.4],& \text{ and }        \lambda_1 + \lambda_2 \in [0.6,1]. \label{eq:cf}
    \end{align}
    The feasible region for $\lambda_1$ and $\lambda_2$ is a 2D polytope, which does not perfectly recover the MC parameters. Note that $\rho^{gt}_{12} = \rho^{gt}_{21} = 0$ while $\rho_{12}$ and $\rho_{21}$ can be positive. The limited number of assortments and the limited information from the choice probability profile $\tilde{\pi}$ do not capture the following scenarios: no customers will transition from state 1 to state 2 when product 1 is not offered, and no customers will transition from state 2 to state 1 when product 2 is not offered.

    In the setting with panel data, we observe that each customer of type $1 \succ 0 \succ 2$ selects product $1$ from $\{0,1\}$ and product $0$ from $\{0,2\}$, so the strict preference $1 \succ 0 \succ 2$ is fully observed. Similarly, if the customer is of type $2 \succ 0 \succ 1$, the strict preference is also fully observed. When the number of customer samples $C$ becomes large, we observe that there are 60 percent of type $1 \succ 0 \succ 2$ customers and 40 percent of type $2 \succ 0 \succ 1$ customers. From \eqref{eq:prob-rle-mc} and \eqref{eq:log-like-cus}, we aim to maximize
    \begin{align}
        & 0.6 \log(\pi(1,\{1,2\})\Pr[1 {\overset{\{1\}}{\leadsto}} 0]) + 0.4 \log(\pi(2,\{1,2\})\Pr[2 {\overset{\{2\}}{\leadsto}} 0]) \nonumber \\
        = \quad & 0.6 \log (\lambda_1 \rho_{10}) + 0.4 \log (\lambda_2 \rho_{20}) = 0.6 \log (\lambda_1 (1-\rho_{12})) + 0.4 \log (\lambda_2 (1-\rho_{21})) \label{eq:ex-cus-llh}
    \end{align}
    subject to $\lambda_1+\lambda_2 \le 1$, and $\lambda_1, \lambda_2, \rho_{12}, \rho_{21} \in [0,1]$. When $\rho_{11}=\rho_{22}=0$, $\pi(1,\{1,2\})\Pr[1 {\overset{\{1\}}{\leadsto}} 0] = \lambda_1 \rho_{10}$ is the probability of having the strict preference $1 \succ 0 \succ 2$ and $\pi(2,\{1,2\})\Pr[2 {\overset{\{2\}}{\leadsto}} 0] = \lambda_2 \rho_{20}$ is the probability of having the strict preference $2 \succ 0 \succ 1$. 
    The optimal solution is $\rho_{12}=\rho_{21}=0$, $\lambda_1=0.6$, and $\lambda_2=0.4$, which recovers the ground truth MC parameters. Note that this solution is also optimal for \eqref{eq:ex-trad-llh}.

    One might wonder how much the objective \eqref{eq:ex-cus-llh} would deteriorate while using an optimal solution to \eqref{eq:ex-trad-llh}. Plugging \eqref{eq:cf} into \eqref{eq:ex-cus-llh}, we have the log-likelihood 
    \begin{align}
        0.6 \log (\lambda_1 (1-\rho_{12})) + 0.4 \log (\lambda_2 (1-\rho_{21}))
        = 0.6 \log (\lambda_1 + \lambda_2 - 0.4) + 0.4 \log (\lambda_1 + \lambda_2 - 0.6). \label{eq:ex-cus-plug}
    \end{align}
    When $\lambda_1 + \lambda_2 \to 0.6^+$, \eqref{eq:ex-cus-plug} reaches $-\infty$. The customer either starts at state 0 with probability 0.4 and ends the choice process, or starts at state 1 (or 2) and transitions from state 1 to state 2 (respectively, state 2 to state 1) with high probability. The estimated MC parameters deviate dramatically from the ground truth.

\subsection{Parameter Identification for MNL} \label{sec:pp}

In this section, we show that the MNL parameters can be identified in the traditional setting when the choice probability vectors of a few assortments are given.

We assume that the choice probability follows an underlying MNL. Suppose a collection of offer sets $\cS = \{S_i \subseteq [n] \mid i \in [m]\}$ is given for some $m > 0$. To ensure that we have information about each product, we assume that $\cS$ is \emph{exhaustive}. That is, all products in $[n]$ are offered, i.e., $\cup_{S \in \cS} S = [n]$. For each $i \in [m]$, the offer set $S_i$ is offered with a given probability $p_i > 0$.

In the traditional setting, the choice probability vectors $\{\tilde\pi(S_+,S)\}_{S \in \cS}$, denoting the probability that $j \in S_+$ is chosen from $S_+$ for all $j \in S_+$, is the given input. Suppose we sample $I$ random transactions independently, where $S_i$ is offered with probability $p_i$, and $j$ is chosen with probability $\tilde\pi(j,S_i)$. When $I$ approaches infinity, the proportion that $S_i$ is offered converges to $p_i$. When $S_i$ is offered, the proportion that $j$ is chosen from ${S_i}_+$ converges to $\tilde\pi(j,S_i)$. Building on \eqref{eq:mnl-prob} and \eqref{eq:log-like-trad}, we aim to solve
\begin{equation} \label{eq:log-like-trad-est}
    \max \sum_{i=1}^m \sum_{j \in {S_i}_+} p_i\tilde\pi(j,S_i) \log (p_i\pi(j,S_i)) = \max_v \sum_{i=1}^m \sum_{j \in {S_i}_+} p_i\tilde\pi(j,S_i) (\log p_i + \log v_j - \log V({S_i}_+))
\end{equation}
by selecting the best MNL parameters. Here, $p_i$ and $\tilde\pi(j,S_i)$ are the input, and the estimated probability $\pi(j,S)$ depends on the MNL parameters $v$ to be selected.

Without loss of generality, we set $v_0=1$. For each $i \in [m]$ , we solve the system of linear equations
\begin{equation} \label{lin-sys:mnl}
    \tilde\pi(j,S_i) (1 + \sum_{\ell \in S_i} v_\ell) = v_j \quad \forall j \in S_i
\end{equation}
to get the value of $v_j$ so that the estimated probability matches $\tilde{\pi}(j,S_i)$. \eqref{lin-sys:mnl} has a unique solution $v_j = \tilde\pi(j,S_i)/\tilde\pi(0,S_i)$ for all $j \in S_i$. We note that if $j \in S_i$ and $j \in S_{i'}$ for $i \ne i'$, then $v_j$ is the solution to both the systems of linear equations of $S_i$ and $S_{i'}$ since the choice probability vectors are consistent with the same underlying MNL. Because $\cS$ is exhaustive, we obtain the value of $v_j$ for all $j \in [n]$. Combining the system of linear equations \eqref{lin-sys:mnl} for all $i \in [m]$, we have the following grand system of linear equations
\begin{equation} \label{lin-sys:mnl-grand}
    \tilde\pi(j,S_i) (1 + \sum_{\ell \in S_i} v_\ell) = v_j \quad \forall j \in S_i, \forall i \in [m].
\end{equation}
We show that the solution to \eqref{lin-sys:mnl-grand} with $v_0=1$ is optimal for \eqref{eq:log-like-trad-est}.

\begin{proposition} \label{lem:mnl-id}
    Given an exhaustive $\cS = \{S_i \subseteq [n] \mid i \in [m]\}$ with $p_i > 0$ for all $i \in [m]$ and choice probability vectors that are consistent with the underlying MNL, the solution to \eqref{lin-sys:mnl-grand} is the unique optimal solution to \eqref{eq:log-like-trad-est} when $v_0$ is fixed to 1. The unique optimal solution does not depend on the $p_i$'s.
\end{proposition}
\begin{proof}
    Let $L^{trad}(v) = \sum_{i=1}^m \sum_{j \in {S_i}_+} p_i\tilde\pi(j,S_i) (\log p_i + \log v_j - \log V({S_i}_+))$. Taking the partial derivative of $L^{trad}$ with respect to $v_j$, we have
    \begin{align*}
        \frac{\partial L^{trad}}{\partial v_j} = 
        \sum_{i \in [m]:j \in S_i} p_i\left(\tilde\pi(j,S_i)\left(\frac{1}{v_j} - \frac{1}{V({S_i}_+)} \right)-\sum_{\ell \in {S_i}_+\setminus\{j\}} \frac{\tilde\pi(\ell,S_i)}{V({S_i}_+)}\right).
    \end{align*}
    Let $v'_j$ be the solution to \eqref{lin-sys:mnl-grand}. Let $V'(S) = \sum_{\ell \in S}v'_\ell$ for any $S \subseteq [n]_+$. For all $i \in [m]$ such that $j \in S_i$, we have $\tilde\pi(j,S_i) = v'_j/V'({S_i}_+)$. Evaluating the derivative at $v_j=v'_j$, we have
    \begin{align*}
        &\tilde\pi(j,S_i)\left(\frac{1}{v'_j} - \frac{1}{V'({S_i}_+)} \right)-\sum_{\ell \in {S_i}_+\setminus\{j\}} \frac{\tilde\pi(\ell,S_i)}{V'({S_i}_+)} 
        =& \frac{v'_j}{V'({S_i}_+)} \frac{1}{v'_j} - \sum_{\ell \in {S_i}_+}\frac{v'_\ell}{V'({S_i}_+)^2} = \frac{1}{V'({S_i}_+)} - \frac{1}{V'({S_i}_+)} = 0
    \end{align*}
    for all $i \in [m]$ such that $j \in S_i$. This implies that $\partial L^{trad} /\partial v_j = 0$ when $v_j = v'_j$ for all $j \in [n]$. Writing \(v_j=\exp(\mu_j)\) and fixing \(\mu_0=0\), $L^{trad}$ is strictly concave in $\mu$. Moreover, we have $\partial L^{trad} /\partial \mu_j = 0$ when $v_j = v'_j$ for all $j \in [n]$ since $\partial L^{trad} /\partial \mu_j = v_j \partial L^{trad} /\partial v_j = 0$ and $v_j > 0$. Since the collection \(\mathcal S\) is exhaustive and every assortment includes the no-purchase option with fixed \(\mu_0\), the solution to \eqref{lin-sys:mnl-grand} is the unique optimal solution that maximizes $L^{trad}$.
\end{proof}

Proposition \ref{lem:mnl-id} implies that when the ground truth is an MNL, and the number of transaction samples is sufficient, then it suffices to gather the choice probability information to identify an MNL. In a scenario where the customer ID is provided for each transaction, ignoring correlations among transactions from the same customer does not hinder recovery of the MNL parameters. 

\section{Markov Chain Model Estimation with Panel Data} \label{sec:em}

In this section, we present EM algorithms for MC parameter estimation with panel data. Our EM algorithms are adapted from the traditional EM algorithm in \cite{csimcsek2018expectation}. In Section \ref{subsec:em-func}, we provide the incomplete and complete likelihood functions used in the EM algorithms. In Section \ref{subsec:em-basic}, we present the EM algorithm \textsc{Cus}, which assumes that each customer has a set of transactions consistent with at least one strict preference, or equivalently, each customer is associated with a reduced DAG. This setting captures the DAG-based partial customer preference information used in \cite{jagabathula2018partial}. In practice, however, individual transactions might not be consistent with any strict preference lists, thus inducing cyclic preferences. In Section \ref{subsec:em-hybrid}, we introduce a hybrid EM algorithm \textsc{Hyb}, which combines the traditional EM algorithm and \textsc{Cus} and provides the flexibility to accommodate transactions that induce cyclic preferences.

\subsection{Incomplete and Complete Likelihood Functions} \label{subsec:em-func}

Recall that $[C]$ is the set of customers and that each customer $c \in [C]$ has a set of transactions $\cD^c:=\{(j^c_\ell, S^c_\ell) \mid \ell \in [k_c]\}$ where $k_c$ is the number of transactions in $\cD^c$. Without loss of generality, we assume that $j^c_\ell \ne j^c_{\ell'}$ for all $\ell \ne \ell'$. Each set of transactions $\cD^c$ is consistent with at least one strict preference list for $[n]_+$, that is, there are no cycles induced by $\cD^c$, and a reduced DAG for $\cD^c$ can be constructed. Unlike the traditional EM algorithm \cite{csimcsek2018expectation}, which assumes that each transaction is independent, we assume that each customer is independent and transactions associated with the same customer are dependent. Our EM algorithm selects MC parameters to maximize the likelihood of the observed customer transaction histories, equivalently represented by reduced DAGs.

Building on \eqref{eq:log-like-cus}, given $\rho$ and $\lambda$, the log-likelihood of having the dataset $\{\cD^c\}_{c \in C}$ is
\begin{equation*} 
    L_I(\lambda,\rho) = \sum_{c=1}^C \log \Pr[\cE_{\cD^c} \mid \lambda, \rho]
\end{equation*}
where the subscript $I$ stands for \emph{incomplete} and $\Pr[\cE_{\cD^c} \mid \lambda, \rho]$ denotes the probability that the event $\cE_{\cD^c}$ occurs given the MC parameters $\rho$ and $\lambda$. Although $\Pr[\cE_{\cD^c} \mid \lambda, \rho]$ has a closed-form expression according to \eqref{prob:i-to-j}, \eqref{eq:prob-rle-mc}, and \eqref{eq:prob-rle}, it is convoluted and requires a better structure to design an EM algorithm. We introduce complete-data variables that record the initial state and transition counts of the unobserved random walk.

Let $F^{c,\sigma}_i \in \{0,1\}$ be a random variable such that $F^{c,\sigma}_i = 1$ if and only if customer $c$ enters the system from state $i$ and its random walk is in the form of the reduced linear extension $\sigma$. For each $\sigma \sim_r \cD^c$ and $s \in [k_c]$, let $U^{\sigma,c}_s:=\cup_{\ell=s}^{k_c} S^c_{\sigma(\ell)}$ and $A^{\sigma,c}_s:=\overline{U^{\sigma,c}_s}$. Let
$X^{c,\sigma,s}_{ij}$ be a random variable defined as follows:
\begin{enumerate}
    \item If the random walk is in the form of $\sigma$:
        \begin{enumerate}
            \item If $s=1$, then $X^{c,\sigma,s}_{ij}$ is the number of times customer $c$ transitions the choice from state $i$ to state $j$ during the course of her choice process starting from the initial state and ending at the first visit of $j_{\sigma(1)}$, through states in $A^{\sigma,c}_1$.
            \item If $s > 1$, then $X^{c,\sigma,s}_{ij}$ is the number of times customer $c$ transitions the choice from state $i$ to state $j$ during the course of her choice process starting from the first visit of $j_{\sigma(s-1)}$ and ending at the first visit of $j_{\sigma(s)}$, through states in $A^{\sigma,c}_s$.
        \end{enumerate}
    \item Otherwise, the random walk is not consistent with $\sigma$, so $X^{c,\sigma,s}_{ij} = 0$.
\end{enumerate}

Given $F^c:=\{F^{c,\sigma}_i\}_{\sigma \sim_r \cD^c, i \in [n]_+}$, $F:=\{F^c\}_{c \in [C]}$, $X^c:=\{X^{c,\sigma,s}_{ij}\}_{\sigma \sim_r \cD^c, s \in [k_c], i \in [n], j \in [n]_+}$, and $X:=\{X^c\}_{c \in [C]}$, we derive a well-structured likelihood function. Suppose the random walk of customer $c$ is in the form of $\sigma$, then for any $\sigma' \ne \sigma$, $F^{c,\sigma'}_i = 0$ and $X^{c,\sigma',s}_{ij} = 0$; besides, the customer enters the system at one of the states in $[n]_+$, so $\sum_{i \in [n]_+} F^{c,\sigma}_i = 1$. This implies that $\sum_{\sigma \sim_r \cD^c} \sum_{i \in [n]_+} F^{c,\sigma}_i = 1$ and the likelihood that customer $c$ has a random walk in the form of $\sigma$ and enters the system from state $i$ is $\prod_{i \in [n]_+}\lambda_i^{F^{c,\sigma}_i}$. After entering state $i$, customer $c$ performs a random walk until the first visit of $j_{\sigma(1)}$, via transitioning from state $i$ to state $j$ for $X^{c,\sigma,1}_{ij}$ times. The likelihood of this sub random walk is $\prod_{i \in [n], j \in [n]_+}\rho_{ij}^{X^{c,\sigma,1}_{ij}}$, regardless of the ordering of the visits in $A^{\sigma,c}_1$. Similarly, for $s>1$, after the first visit of $j_{\sigma(s-1)}$, customer $c$ performs a random walk until the first visit of $j_{\sigma(s)}$, via transitioning from state $i$ to state $j$ for $X^{c,\sigma,s}_{ij}$ times, regardless of the ordering of the visits in $A^{\sigma,c}_s$. The likelihood of this sub random walk is $\prod_{i \in [n], j \in [n]_+}\rho_{ij}^{X^{c,\sigma,s}_{ij}}$. By the Markov property, conditional on customer $c$’s random walk being in the form of $\sigma$, the likelihood factors over the prefix segment and the subsequent inter-arrival segments. Therefore, the likelihood that customer $c$ has a random walk in the form of $\sigma$, $F^c$, and $X^c$ is $\left(\prod_{i \in [n]_+}\lambda_i^{F^{c,\sigma}_i}\right) \prod_{s=1}^{k_c}\prod_{i \in [n],
j \in [n]_+
}\rho_{ij}^{X^{c,\sigma,s}_{ij}}$. 

Taking into account all customers $c \in [C]$ and all possible reduced linear extensions $\sigma$, the likelihood of having the dataset $\{\cD^c\}_{c \in [C]}$ that is consistent with $F$ and $X$ is
\[\prod_{c=1}^C\prod_{\sigma \sim_r \cD^c} \left(\left(\prod_{i \in [n]_+}\lambda_i^{F^{c,\sigma}_i}\right) \prod_{s=1}^{k_c}\prod_{
i \in [n]}
\prod_{j \in [n]_+
}\rho_{ij}^{X^{c,\sigma,s}_{ij}}\right),\]
and the log-likelihood is
\begin{equation} \label{eq:lc}
    L_C(\lambda,\rho) = \sum_{c=1}^C \sum_{\sigma \sim_r \cD^c} \left(\sum_{i \in [n]_+} F^{c,\sigma}_i \log \lambda_i +  \sum_{s=1}^{k_c} \sum_{
i \in [n]}
\sum_{j \in [n]_+}  X^{c,\sigma,s}_{ij} \log \rho_{ij}\right).
\end{equation}
The log-likelihood function \eqref{eq:lc} is written in terms of the latent variables $F$ and $X$, where the subscript $C$ stands for \emph{complete}. Since $\log x$ is concave in $x$, $L_C$ is concave in $\lambda$ and $\rho$. However, this formulation is not yet useful since we do not have access to $F$ and $X$. Nevertheless, the EM algorithm \cus presented in the next section replaces these latent variables by their conditional expectations.

\subsection{The EM Algorithm \textsc{Cus}} \label{subsec:em-basic}

In this section, we describe the EM algorithm \textsc{Cus}. An overview is provided in Section \ref{subsec:em-ov}. The expectation step and the maximization step are presented in Sections \ref{subsec:em-ex} and \ref{subsec:em-max}, respectively. 

\subsubsection{Overview of the Algorithm} \label{subsec:em-ov}

Let $\lambda^t$ and $\rho^t$ be the estimated MC parameters at the end of iteration $t$. Let $\lambda^0$ and $\rho^0$ be the initial MC parameters. Each iteration consists of an expectation step and a maximization step.

In the expectation step of iteration $t$, we first estimate the expected value of $F^{c,\sigma}_i$ and $X^{c,\sigma,s}_{ij}$, given the event $\cE_{\cD}$ denoting that customer $c$'s transaction set is $\cD^c$, and the MC parameters $\lambda^{t-1}$ and $\rho^{t-1}$. The estimation for $F^{c,\sigma}_i$ and $X^{c,\sigma,s}_{ij}$ at iteration $t$ are denoted as 
\[F^{c,\sigma,t}_i = \E[F^{c,\sigma}_i \mid \cE_{\cD^c},\lambda^{t-1},\rho^{t-1}] \text{ and } X^{c,\sigma,s,t}_{ij} = \E[X^{c,\sigma,s}_{ij} \mid \cE_{\cD^c},\lambda^{t-1},\rho^{t-1}], \text{ respectively}.\] Let $F^t:=\{F^{c,\sigma,t}_i\}_{c \in [C], \sigma \sim_r \cD^c, i \in [n]_+}$ and $X^t = \{X^{c,\sigma,s,t}_{ij}\}_{c \in [C], \sigma \sim_r \cD^c, s \in [k_c], i \in [n], j \in [n]_+}$. The expected value of $L_C(\lambda^{t-1},\rho^{t-1})$, conditional on $\{\cE_{\cD}\}_{c \in [C]}$ and the MC parameters $\lambda^{t-1}$ and $\rho^{t-1}$, can be obtained by plugging $F^t$ and $X^t$ in \eqref{eq:lc}.

In the maximization step of iteration $t$, given $F^t$ and $X^t$, we select $\lambda^t$ and $\rho^t$ as the solution that maximizes the log-likelihood objective \eqref{eq:lc}.

If the difference between $(\lambda^t, \rho^t)$ and $(\lambda^{t-1}, \rho^{t-1})$ is within tolerance, we end the EM algorithm. The algorithm repeats the expectation and maximization steps until convergence.\footnote{The proof of convergence closely follows \cite{csimcsek2018expectation,nettleton1999convergence} hence omitted.}

The EM algorithm \cus is briefly outlined as follows.

\begin{description}
    \item[\bf Step 1 (Initialization):] Initialize $\lambda^0$ and $\rho^0$ where $\sum_{i = 0}^n\lambda^0_i = 1$ and $\sum_{j = 0}^n\rho^0_{ij} = 1$ for all $i \in [n]$. 
    \item[\bf Step 2 (Expectation):] At iteration $t$, given $\lambda^{t-1}$ and $\rho^{t-1}$, set $F^{c,\sigma,t}_i = \E[F^{c,\sigma}_i \mid \cE_{\cD^c},\lambda^{t-1},\rho^{t-1}]$ and $X^{c,\sigma,s,t}_{ij} = \E[X^{c,\sigma,s}_{ij} \mid \cE_{\cD^c},\lambda^{t-1},\rho^{t-1}]$.
    \item[\bf Step 3 (Maximization):] Given $F^t$ and $X^t$, let $\lambda^t, \rho^t$ be the optimal solution of
    \begin{equation} \label{opt:em-max}
    \max_{\lambda, \rho} L_C(\lambda,\rho) \text{ over } \lambda \in \R^{n+1}_{\ge 0}, \rho \in \R^{n \times (n+1)}_{\ge 0} \text{ s.t. } \sum_{i \in [n]_+}\lambda_i = 1 \text{ and } \sum_{j \in [n]_+}\rho_{ij} = 1 \quad \forall i \in [n].
\end{equation}
\item[\bf Step 4 (Convergence Check):] If the difference between $(\lambda^t, \rho^t)$ and $(\lambda^{t-1}, \rho^{t-1})$ is within tolerance, terminate the algorithm. Otherwise, move on to the next iteration and go to Step 2.
\end{description}

We emphasize that \textsc{Cus} is more computationally intensive than the traditional EM algorithm \cite{csimcsek2018expectation} due to the blowup in the number of random variables. Nevertheless, our experimental results in Section \ref{sec:ex} show that \cus is practically feasible when the number of transactions per customer and the number of reduced linear extensions per customer are at a reasonable scale.

\subsubsection{The Expectation Step} \label{subsec:em-ex}

In this section, we compute $\E[F^{c,\sigma}_i \mid \cE_{\cD^c},\lambda^{t-1},\rho^{t-1}]$ and $\E[X^{c,\sigma,s}_{ij} \mid \cE_{\cD^c},\lambda^{t-1},\rho^{t-1}]$. For notation brevity, we focus on one specific customer and drop the notation $c$ for the customer. We compute the expected value of $F^{\sigma}_i$ and $X^{\sigma,s}_{ij}$ when the customer has a transaction set $\cD$ and the parameters are $\lambda^{t-1}$ and $\rho^{t-1}$.

\paragraph{\bf Computation of $\E[F^{\sigma}_i \mid \cE_{\cD},\lambda^{t-1},\rho^{t-1}]$.} Suppose the random walk is in the form of a reduced linear extension $\sigma \sim_r \cD$. Then, we cannot have any state in $U:=\cup_{\ell=1}^{k} S_{\ell}$ as the initial state when the customer enters the system, except $j_{\sigma(1)}$. For $\sigma \sim_r \cD$, recall the computation of $\pi^\cD_\sigma$ in \eqref{eq:prob-rle-mc}. Let
    \begin{equation} \label{eq:p-sigma}
        p_\cD(\sigma):=\prod_{s = 1}^{k-1}\Pr[j_{\sigma(s)} {\overset{A^\sigma_{s+1}}{\leadsto}} j_{\sigma(s+1)}]
    \end{equation}
    denote the probability of having a random walk in the form of $j_{\sigma(1)}\overset{A^\sigma_2}{\leadsto} j_{\sigma(2)} \overset{A^\sigma_3}{\leadsto} ...  \overset{A^\sigma_k}{\leadsto} j_{\sigma(k)}$ after the first visit of $j_\sigma(1)$. Let $q_\cD(i,\sigma)$ be the probability that the random walk is in the form of $\sigma$ and the initial state is $i$. We have that
\begin{equation} \label{eq:q-cD}
    q_\cD(i,\sigma) = \begin{cases}
        \lambda_i \Pr[i \overset{\overline{U}}{\leadsto} j_{\sigma(1)}] p_\cD(\sigma) & \text{if } i \in  \overline{U}, \\
        \lambda_i p_\cD(\sigma) & \text{if } i=j_{\sigma(1)}, \\
        0 & \text{otherwise},
    \end{cases}
\end{equation}
where the parameters $(\lambda, \rho)$ in \eqref{eq:q-cD} are replaced with $\lambda^{t-1}$ and $\rho^{t-1}$, respectively. Conditional on $\cE_{\cD}$, if the random walk is not in the form of $\sigma$, then $F^{\sigma}_i = 0$; otherwise, the expected value of $F^{\sigma}_i$ is the probability that $i$ is the initial state and the random walk is in the form of $\sigma$. Hence,
\begin{equation*} 
    \E[F^{\sigma}_i \mid \cE_{\cD}, \lambda^{t-1}, \rho^{t-1}] = 
        \frac{q_{\cD}(i,\sigma)}{\pi_{\cD}},
\end{equation*}
where $q_{\cD}(i,\sigma)$ in \eqref{eq:q-cD} and $\pi_{\cD}$ in \eqref{eq:prob-rle} have closed-form expressions in terms of the transaction set $\cD$ and the MC parameters $\lambda^{t-1}$ and $\rho^{t-1}$.

\paragraph{\bf Computation of $\E[X^{\sigma,s}_{ij} \mid \cE_{\cD},\lambda^{t-1},\rho^{t-1}]$.} If the random walk is not in the form of $\sigma$, then $X^{\sigma,s}_{ij} = 0$. We consider the event that the random walk is in the form of $\sigma$. Conditional on $\cE_{\cD}$, the probability that the random walk is in the form of $\sigma$ is $\pi(j_{\sigma(1)}, U) p_\cD(\sigma) / \pi_\cD$ as $U$ is first-visited at $j_{\sigma(1)}$ and the remaining is in the form of the suffix of $\sigma$ after the first visit of $j_{\sigma(1)}$, which occurs with probability $p_\cD(\sigma)$.

We are interested in computing the expected number of times that the customer transitions the choice from state $i$ to state $j$ during the course of her choice process starting from the first visit of $j_{\sigma(s-1)}$ and ending at the first visit of $j_{\sigma(s)}$, through states in $A^\sigma_s$, when $s > 1$. When $s=1$, we compute the number of times the choice transitions from state $i$ to state $j$ during the course of the choice process from the initial state to the first visit of $j_\sigma(1)$. To account for $s=1$ and $s>1$, we consider a more general problem.

Let $\eta := \{\eta_i\}_{i \in [n]_+} \in \R^{n+1}_{\ge 0}$ with $\sum_{i \in [n]_+} \eta_i = 1$ denote a distribution for the starting state of a random walk. For example, when the customer enters the system, $\eta = \lambda$; right after the customer's first visit of $j_{\sigma(s)}$, $\eta_{j_{\sigma(s)}} = 1$ and $\eta_j = 0$ for $j \in [n]_+ \setminus \{j_{\sigma(s)}\}$. Let $A \subseteq [n]$ be the set of states that are allowed to be visited before arriving at any state in $\overline{A}_+$. Consider the following system of linear equations adapted from \eqref{lineq:mc}:
\begin{equation*}
    \theta_{\eta}(j,\overline{A}) = \eta_j + \sum_{i \in A} \rho_{ij} \theta_{\eta}(i,\overline{A}) \quad \forall j \in [n]_+.
\end{equation*}
Following the same argument as in Section \ref{subsubsec:mc}, $\theta_{\eta}(j,\overline{A})$ denotes the expected number of times that state $j$ is visited during the course of the choice process when the starting state distribution is $\eta$ and $\overline{A}$ is offered. For $T \subseteq [n]_+$, let $\eta(T):=\{\eta_j\}_{j \in T}$ denote a sub-vector of $\eta$ with coordinates in $T$, and let $\theta_{\eta}(T,\overline{A}):=\{\theta_{\eta}(j, \overline{A})\}_{j \in T}$. Following \eqref{eq:pi-bar-S}, the closed-form expression for $\theta_{\eta}(A,\overline{A})$ is
\begin{equation*}
    \theta_{\eta}(A,\overline{A}) = \left((I-\rho(A, A))^T\right)^{-1} \eta(A) = \left((I-\rho(A, A))^{-1}\right)^T \eta(A).
\end{equation*}
Following \eqref{eq:pi-S}, the closed-form expression for $\theta_{\eta}(\overline{A}_+,\overline{A})$ is
\begin{align*}
    \theta_\eta(\overline{A}_+,\overline{A}) = \eta(\overline{A}_+) + \rho(A, \overline{A}_+)^T \left((I-\rho(A, A))^{-1}\right)^T \eta(A).
\end{align*}

Let $Y_{ij}$ be a random variable that denotes the number of times the choice transitions from state $i$ to state $j$ during the choice process starting from a given state distribution $
\eta$ and ending at a state in $\overline{A}_+$. Let $d \in \overline{A}_+$ be a destination state. Let $\cF_{\eta}(d,\overline{A})$ be the event that the random walk enters at state $d$ upon the first visit of $\overline{A}_+$, starting from distribution $\eta$. Given the starting state distribution $\eta$ and the offer set $\overline{A}$, conditional on $d$ being chosen from $\overline{A}_+$, a closed-form expression for the expectation of $Y_{ij}$ is as follows.

\begin{lemma}[\cite{csimcsek2018expectation}] \label{lem:yij}
    \[\E[Y_{ij} \mid \cF_{\eta}(d,\overline{A})] = \frac{\Pr[j \overset{A}{\leadsto} d] \rho_{ij} \theta_{\eta}(i,\overline{A})}{\theta_{\eta}(d,\overline{A})}.\]
\end{lemma}

Equipped with Lemma \ref{lem:yij}, we compute $\E[X^{\sigma,s}_{ij} \mid \cE_{\cD},\lambda^{t-1},\rho^{t-1}]$. Recall that when the random walk is in the form of $\sigma$, it is partitioned into sub random walks, starting from an initial state to the first visit of $j_{\sigma(1)}$, followed by the first visit of $j_{\sigma(s-1)}$ to the first visit of $j_{\sigma(s)}$ for $s > 1$. By the Markov property, conditional on the endpoints of each segment and the reduced linear extension $\sigma$, the likelihood factors into the product of the segment transition probabilities. Let $e^j$ denote a state distribution vector such that $e^j_j = 1$ for $j \in [n]$ and $e^j_{j'} = 0$ for $j' \in [n]_+ \setminus \{j\}$. We have the following closed-form expression for $\E[X^{\sigma,s}_{ij} \mid \cE_{\cD},\lambda^{t-1},\rho^{t-1}]$:
\begin{equation} \label{eq:ex-x}
    \E[X^{\sigma,s}_{ij} \mid \cE_{\cD},\lambda^{t-1},\rho^{t-1}] = \begin{cases}
        \mbox{\Large $\frac{\pi^\cD_\sigma}{\pi_\cD} \frac{\Pr[j \overset{\overline{U}}{\leadsto} j_{\sigma(1)}] \rho_{ij} \theta_{\lambda}(i,U)}{\theta_{\lambda}(j_{\sigma(1)},U)}$} & \text{if } s=1, \\
        \mbox{\Large $\frac{\pi^\cD_\sigma}{\pi_\cD} \frac{\Pr[j \overset{A^\sigma_s}{\leadsto} j_{\sigma(s)}] \rho_{ij} \theta_{e^{j_{\sigma(s-1)}}}(i,U^\sigma_s)}{\theta_{e^{j_{\sigma(s-1)}}}(j_{\sigma(s)},U^\sigma_s)}$} & \text{if } s \in [k] \setminus \{1\}.
    \end{cases}
\end{equation}
The MC parameters in \eqref{eq:ex-x} are $\lambda^{t-1}$ and $\rho^{t-1}$. The term $\pi^{\cD}_{\sigma} / \pi_\cD$ is the probability that the random walk is in the form of $\sigma$. The random variable $X^{\sigma,s}_{ij}$ is captured by the random variable $Y_{ij}$ in Lemma \ref{lem:yij}. When $s=1$, the starting distribution is $\lambda$, the admissible states are $\overline{U}$, and the destination state is $j_{\sigma(1)}$. When $s>1$, the starting distribution is $e^{j_{\sigma(s-1)}}$, the admissible states are $A^\sigma_s$, and the destination state is $j_{\sigma(s)}$.

\subsubsection{The Maximization Step} \label{subsec:em-max} After computing $F^t$ and $X^t$ in the expectation step, we aim to solve \eqref{opt:em-max}, which can be decomposed into the following optimization sub-problems:
\begin{equation} \label{opt:em-max-lambda}
\max_{\lambda} \sum_{c=1}^C \sum_{\sigma \sim_r \cD^c} \sum_{i \in [n]_+} F^{c,\sigma}_i \log \lambda_i \quad \text{over} \quad 
\lambda \in \R^{n+1}_{\ge 0} \quad \text{s.t.} \quad \sum_{i \in [n]_+}\lambda_i = 1
\end{equation}
and
\begin{equation} \label{opt:em-max-rho}
    \max_{\rho} \sum_{c=1}^C \sum_{\sigma \sim_r \cD^c} \sum_{s=1}^{k_c} \sum_{
i \in [n]}
\sum_{j \in [n]_+} X^{c,\sigma,s}_{ij} \log \rho_{ij} \text{ over } \rho \in \R^{n \times (n+1)}_{\ge 0} \quad \text{s.t.} \sum_{j \in [n]_+}\rho_{ij} = 1 \quad \forall i \in [n].
\end{equation}
\eqref{opt:em-max-lambda} and \eqref{opt:em-max-rho} are in the form of the problem of computing the maximum likelihood estimators under the multinomial distribution. According to \cite{bishop2006pattern}, the optimal solution $x^*$ of the problem
\begin{equation} \label{opt:multi}
    \max_{x} \sum_{i=1}^m a_i \log x_i \quad \text{over} \quad x \in \R^m_{\ge 0} \quad \text{s.t.} \quad \sum_{i=1}^m x_i = 1
\end{equation}
where $a \in \R^m_{\ge 0}$ and $\sum_{i=1}^m a_i > 0$, is unique and has a closed-form expression  $x^*_i = a_i / (\sum_{i=1}^m a_i)$ for all $i \in [m]$.\footnote{For ease of notation, when $a_i=0$, then $x_i=0$, and we set $a_i \log x_i = 0$ in the objective of \eqref{opt:multi}.} Therefore, we set
\begin{equation*}
    \lambda^t_i = \frac{\sum_{c=1}^C \sum_{\sigma \sim_r \cD^c} F^{c,\sigma,t}_i}{\sum_{c=1}^C \sum_{\sigma \sim_r \cD^c}\sum_{j \in [n]_+} F^{c,\sigma,t}_j} \quad \forall i \in [n]_+
\end{equation*}
and
\begin{equation*}
    \rho^t_{ij} = \frac{\sum_{c=1}^C \sum_{\sigma \sim_r \cD^c} \sum_{s=1}^{k_c} X^{c,\sigma,s,t}_{ij}}{\sum_{c=1}^C \sum_{\sigma \sim_r \cD^c} \sum_{s=1}^{k_c} \sum_{k \in [n]_+} X^{c,\sigma,s,t}_{ik}} \quad \forall i \in [n], j \in [n]_+.
\end{equation*}

\subsection{The EM Algorithm \hyb} \label{subsec:em-hybrid} In addition to the customer transactions $\{\cD^c\}_{c \in [C]}$ where each customer's transaction set is consistent with at least one strict preference, the hybrid EM algorithm \hyb also takes the \emph{independent transaction set} $\cI$ as the input. The independent transaction set $\cI = \{(j_r, S_r) \mid r \in [I]\}$ has $I$ transactions $(j_r, S_r)$ denoting that $j_r$ is chosen from the offer set $S_r$ in transaction $r$. Each transaction in $I$ is assumed to be independent. \hyb aims to maximize the likelihood of having both $\{\cD^c\}_{c \in [C]}$ and $\cI$, by selecting the MC parameters. 
If all transactions are added to $\cI$, then \hyb captures the traditional EM algorithm.

In practice, the independent transaction set provides the flexibility to adapt to real-life datasets. Suppose a customer has a transaction set that induces a cyclic preference. One heuristic is to construct an arc-weighted directed graph that captures the customer's preference according to the transactions, and extract a DAG that best describes the customer's preference. The DAG construction might involve arc deletions. Transactions that violate the DAG preference are then not used in the DAG construction, but could provide information about the preference distribution. As opposed to the DAG-based heuristic in \cite{jagabathula2018partial} that discards the information from the violating arcs, one approach is to include these DAG-violating transactions in $\cI$. Another heuristic is to consider the most recent transactions in the DAG construction, and include the older transactions and the DAG-violating transactions in $\cI$.


Let $G^r_i \in \{0,1\}$ be a random variable such that $G^r_i = 1$ if and only if the random walk of transaction $r$ enters the system from state $i$. Let $Y^r_{ij}$ be a random variable denoting the number of times the choice transitions from state $i$ to state $j$ during the course of the choice process in the random walk of transaction $r$. Let $G:=\{G^r_i\}_{r \in [I], i \in [n]_+}$ and $Y:=\{Y^r_{ij}\}_{r \in [I], i \in [n], j \in [n]_+}$. Following \cite{csimcsek2018expectation}, the likelihood of having $\cI$ is
\[\prod_{r=1}^I \left(\left(\prod_{i \in [n]_+}\lambda_i^{G^r_i}\right) \prod_{\substack{
i \in [n], \\
j \in [n]_+
}}\rho_{ij}^{Y^r_{ij}}\right),\]
and the log-likelihood is
\begin{equation} \label{eq:L}
    L(\lambda,\rho) = \sum_{r=1}^I \left(\sum_{i \in [n]_+} G^r_i \log \lambda_i +  \sum_{
i \in [n]}
\sum_{j \in [n]_+
} Y^r_{ij} \log \rho_{ij}\right).
\end{equation}
The log-likelihood of having both $\{\cD^c\}_{c \in [C]}$ and $\cI$, called the \emph{hybrid log-likelihood}, is 
\[L_H(\lambda,\rho) = L_C(\lambda,\rho) + L(\lambda,\rho).\]

The hybrid EM algorithm \hyb is briefly outlined as follows.

\begin{description}
    \item[\bf Step 1 (Initialization):] Initialize $\lambda^0$ and $\rho^0$ where $\sum_{i=0}^n\lambda^0_i = 1$ and $\sum_{j=0}^n\rho^0_{ij} = 1$ for all $i \in [n]$. 
    \item[\bf Step 2 (Expectation):] At iteration $t$, given $\lambda^{t-1}$ and $\rho^{t-1}$, set $F^{c,\sigma,t}_i = \E[F^{c,\sigma}_i \mid \cE_{\cD^c},\lambda^{t-1},\rho^{t-1}]$, $X^{c,\sigma,s,t}_{ij} = \E[X^{c,\sigma,s}_{ij} \mid \cE_{\cD^c},\lambda^{t-1},\rho^{t-1}]$,  $G^{r,t}_i = \E[G^r_i \mid \cE(j_r,S_r),\lambda^{t-1},\rho^{t-1}]$, and $Y^{r,t}_{ij} = \E[Y^r_{ij} \mid \cE(j_r,S_r),\lambda^{t-1},\rho^{t-1}]$, where $\cE(j_r,S_r)$ denotes the event that $j_r$ is chosen from ${S_r}_+$.
    \item[\bf Step 3 (Maximization):] Given $F^t$, $X^t$, $G^t$, and $Y^t$, let $\lambda^t, \rho^t$ be the optimal solution of
    \begin{equation} \label{opt:em-max-hybrid}
    \max_{\lambda, \rho} L_H(\lambda,\rho) \text{ over } \lambda \in \R^{n+1}_{\ge 0}, \rho \in \R^{n \times (n+1)}_{\ge 0} \text{ s.t. } \sum_{i \in [n]_+}\lambda_i = 1 \text{ and } \sum_{j \in [n]_+}\rho_{ij} = 1 \quad \forall i \in [n].
\end{equation}
\item[\bf Step 4 (Convergence Check):] If the difference between $(\lambda^t, \rho^t)$ and $(\lambda^{t-1}, \rho^{t-1})$ is within tolerance, terminate the algorithm. Otherwise, move on to the next iteration and go to Step 2.
\end{description}

In the expectation step, the calculation for $F$ and $X$ follows Section \ref{subsec:em-ex}. The computation for $G$ and $Y$ follows the same idea in \cite{csimcsek2018expectation}. The term for $G^{r,t}_i$ is as follows:
\[\E[G^r_i \mid \cE(j_r,S_r),\lambda^{t-1},\rho^{t-1}] = \frac{\lambda_i \Pr[i \overset{ \overline{S_r}}{\leadsto} j_r]}{\pi(j_r,S_r)}\]
since when the random walk is consistent with the transaction $(j_r, S_r)$, if it enters the system at state $i$, it must be followed by a random walk in the form of $i \overset{ \overline{S_r}}{\leadsto} j_r$. The term for $Y^{r,t}_{ij}$ is as follows:
\[\E[Y^r_{ij} \mid \cE(j_r,S_r),\lambda^{t-1},\rho^{t-1}] = \frac{\Pr[j \overset{ \overline{S_r}}{\leadsto} j_r] \rho_{ij} \theta_{\lambda}(i,S_r)}{\theta_{\lambda}(j_r,S_r)}\]
where the term $\theta_\lambda$ follows Lemma \ref{lem:yij}. The starting state distribution is $\lambda$, the destination state is $j_r$, and the admissible set is $\overline{S_r}$.

In the maximization step, the optimization problem \eqref{opt:em-max-hybrid} is in the form of \eqref{opt:multi}. Therefore, we set
\begin{equation*}
    \lambda^t_i = \frac{\sum_{c=1}^C \sum_{\sigma \sim_r \cD^c} F^{c,\sigma,t}_i + \sum_{r=1}^I G^{r,t}_i}{\sum_{c=1}^C \sum_{\sigma \sim_r \cD^c}\sum_{j \in [n]_+} F^{c,\sigma,t}_j + \sum_{r=1}^I \sum_{j \in [n]_+} G^{r,t}_j} \quad \forall i \in [n]_+
\end{equation*}
and
\begin{equation*}
    \rho^t_{ij} = \frac{\sum_{c=1}^C \sum_{\sigma \sim_r \cD^c} \sum_{s=1}^{k_c} X^{c,\sigma,s,t}_{ij} + \sum_{r=1}^I Y^{r,t}_{ij}}{\sum_{c=1}^C \sum_{\sigma \sim_r \cD^c} \sum_{s=1}^{k_c} \sum_{k \in [n]_+} X^{c,\sigma,s,t}_{ik} + \sum_{r=1}^I \sum_{k \in [n]_+} Y^{r,t}_{ik}} \quad \forall i \in [n], j \in [n]_+.
\end{equation*}

\section{Experimental Results} \label{sec:ex}

To evaluate the performance of our EM algorithms, we design experiments where the ground truth choice models are known. 
In Section \ref{subsec:pm-known}, we introduce the performance measures. In Section \ref{subsec:ben}, we introduce the benchmark choice models that we compare with. In Section \ref{subsec:syn}, we present the experimental results on pure synthetic datasets. In Section \ref{subsec:sushi}, we present the experimental results on the semi-synthetic data constructed from the sushi dataset \cite{kamishima2003nantonac}. These experiments are adapted from \cite{berbeglia2022comparative}. 


\subsection{Performance Measures with Ground Truth} \label{subsec:pm-known}

We consider two types of performance measures: predictive performance (unconditional and conditional soft root mean square error), and revenue-based measures (revenue ratio and conditional revenue ratio).

To evaluate predictive performance, we use root mean square error (RMSE) to capture how far an estimated MC model is from the ground truth. Let $\theta$ be the ground truth model and $\phi$ be an estimated model. Let $\pi^\psi(j,S)$ denote the unconditional probability $\pi(j,S)$ that $j$ is chosen from $S_+$ under model $\psi$, where $\psi$ could be $\theta$ or $\phi$. Following \cite{berbeglia2022comparative}, the \emph{soft RMSE} is defined as follows:
\begin{equation} \label{eq:srmse}
    \sqrt{\frac{\sum_{S \subseteq [n]}\sum_{j \in S^+} \left(\pi^\phi(j,S) - \pi^\theta(j,S)\right)^2}{\sum_{S \subseteq [n]}(|S|+1)}}.
\end{equation}
Similarly, using \eqref{eq:cond-prob}, let $\pi^\psi_{\cD^c}(j,S)$ denote the probability that $j$ is chosen from $S_+$ under model $\psi$, conditional on the transaction data $\cD^c$ of customer $c$. Note that $\pi^\psi_{\cD^c}(j,S) = \pi^\psi(j,S)$ if customer $c$ has no historical transaction, i.e., $\cD^c = \varnothing$. The \emph{conditional soft RMSE} is defined as follows:
\begin{equation} \label{eq:srmse-con}
    \sqrt{\frac{\sum^C_{c=1}\sum_{S \subseteq [n]}\sum_{j \in S^+} \left(\pi^\phi_{\cD^c}(j,S) - \pi^\theta_{\cD^c}(j,S)\right)^2}{\sum_{c=1}^C\sum_{S \subseteq [n]}(|S|+1)}}.
\end{equation}
In \eqref{eq:srmse} and \eqref{eq:srmse-con}, the RMSE is computed over all possible offer sets and all possible selections for each offer set, which ensures robustness. Additionally, in \eqref{eq:srmse-con}, the RMSE is computed over all customers that are present in the data.\footnote{Considering all possible transaction sets (or DAGs), which have $2^ {\Theta(n^2)}$ possibilities, is not computationally practical.} The additive 1 in the denominator represents the no-purchase option.

Building on predictive performance, we proceed with conditional and unconditional revenue performance. Given the revenue vector $r$, under model $\psi$, let $S^*_{\psi}$ denote the optimal assortment for \eqref{opt:tao} and $R_{\psi}(S) := \sum_{j \in S} r_j \pi^\psi(j,S)$ denote the expected revenue obtained by offering $S$. The \emph{revenue ratio} is defined as $R_{\theta}(S^*_\phi)/R_{\theta}(S^*_\theta)$. That is, under the ground truth $\theta$, the ratio between the expected revenue of the optimal assortment under $\phi$ and that under the ground truth. Under $\psi$, let $S^*_{\psi,\cD^c}$ denote the optimal assortment for \eqref{opt:cao} and $R_{\psi, \cD^c}(S):= \sum_{j \in S} r_j \pi^\psi_{\cD^c}(j,S)$ denote the expected revenue obtained by offering $S$ conditional on $\cD^c$. The \emph{conditional revenue ratio} is
\begin{equation} \label{eq:rev-con}
    \frac{\sum_{c=1}^C R_{\theta,\cD^c}(S^*_{\phi,\cD^c})}{\sum_{c=1}^C R_{\theta,\cD^c}(S^*_{\theta,\cD^c})}.
\end{equation}
That is, under the ground truth $\theta$, we consider the ratio between the total expected revenue of the optimal assortment conditional on $\cD^c$ under $\phi$ and that under the ground truth $\theta$.

\subsection{Benchmark Choice Models} \label{subsec:ben}

In addition to the MC choice model estimated by the traditional EM algorithm \cite{csimcsek2018expectation}, denoted \textsc{Mc}, we compare our estimated choice models with two types of choice models: MNL and latent-class MNL (LC-MNL). The LC-MNL model is a mixture of $\ell$ MNL models, each capturing a customer segment. Each segment $s \in [\ell]$ is associated with a probability $\alpha_s$, and each customer belongs to a segment with probability $\alpha_s$. Each type can be further categorized as traditional or partial-ordering-based \cite{jagabathula2018partial}.

Traditional choice model estimation algorithms aim to find the parameters to maximize log-likelihood \eqref{eq:log-like-trad}. For MNL, when the parameter $v_0$ is fixed, \eqref{eq:log-like-trad} is concave and can be efficiently optimized
using standard nonlinear optimization packages. For LC-MNL, \eqref{eq:log-like-trad} becomes non-concave. To ensure comparability, we consider the EM algorithm \cite{train2008algorithms}, as the partial-ordering-based LC-MNL estimation in \cite{jagabathula2018partial} is also based on the framework \cite{train2008algorithms}. In the traditional setting, the estimated MNL and LC-MNL models are termed \textsc{Mnl} and \textsc{Lc-Mnl}, respectively.

For choice model estimation with panel data, the log-likelihood objective \eqref{eq:log-like-cus} is more complicated. The objective is non-concave under MNL. We proceed with the partial-ordering-based MNL estimation. Following \cite{jagabathula2018partial} and Section \ref{subsec:dag}, consider the DAG $D_c=([n]_+,A_c)$ of customer $c$. Let $\Psi^c_j$ denote the set of products that are reachable from product $j$ in $D_c$, including $j$ itself. The probability $\pi_{\cD^c}$ can be approximated as 
\begin{equation} \label{eq:apx-cd}
    \hat{\pi}_{\cD^c} := \prod_{j \in [n]_+} \frac{v_j}{\sum_{j' \in \Psi^c_j} v_{j'}}.
\end{equation}
It was shown that $\hat{\pi}_{\cD^c} \le \pi_{\cD^c}$ where the equality holds when $D_c$ is a directed forest. A lower bound for $\hat{\pi}_{\cD^c}$ in terms of $\pi_{\cD^c}$ was presented in \cite{jagabathula2022personalized}. Building on \eqref{eq:apx-cd}, the approximate log-likelihood is
\begin{equation} \label{eq:apx-llh}
    \hat{L}^{cus}(v) := \sum_{c=1}^C \log \left( \prod_{j \in [n]_+} \frac{v_j}{\sum_{j' \in \Psi^c_j} v_{j'}} \right) = \sum_{c=1}^C \log \hat{\pi}_{\cD^c}.
\end{equation}
The objective \eqref{eq:apx-llh}, when $v_0$ is fixed, is concave in $\mu_j = \log v_j$ for all $j \in [n]$ and can be efficiently optimized
using standard nonlinear optimization packages. We focus on the simple \emph{partial order MNL} that maximizes \eqref{eq:apx-llh} as our benchmark, as opposed to the more sophisticated estimation in \cite{jagabathula2018partial,jagabathula2022personalized}, which considers $\ell_1$-regularization. The benchmark partial order MNL model is termed \textsc{Pomnl}.

Building on the approximation \eqref{eq:apx-cd}, we turn to the \emph{latent-class partial order MNL} (LC-POMNL) model introduced in \cite{jagabathula2018partial}. Here, the probability $\pi_{\cD^c}$ can be approximated as 
\begin{equation} \label{eq:apx-cd-lc}
    \hat{\pi}_{\cD^c} := \sum_{s=1}^\ell \alpha_s \prod_{j \in [n]_+} \frac{v^s_j}{\sum_{j' \in \Psi^c_j} v^s_{j'}}
\end{equation}
where $v^s$ is the MNL attraction vector of segment $s$. Building on \eqref{eq:apx-cd-lc}, the goal is to maximize the approximate log-likelihood objective by solving
\begin{align} \label{eq:apx-llh-lc}
    \max_{\alpha,v} \hat{L}^{cus}(\alpha,v) &:= \sum_{c=1}^C \log \left( \sum_{s=1}^\ell \alpha_s \prod_{j \in [n]_+} \frac{v^s_j}{\sum_{j' \in \Psi^c_j} v^s_{j'}} \right) \\ 
    \text{s.t.} & \sum_{s=1}^\ell \alpha_s = 1, \alpha \in \R^\ell_{\ge 0}, \nonumber \\
    &v \in \R^{\ell \times (n+1)}_{\ge 0} \text{ with } v^s_0=1 \quad \forall s \in [\ell]. \nonumber
\end{align}
We consider the EM algorithm based on \cite{train2008algorithms}, which aims to solve \eqref{eq:apx-llh-lc}. This simple partial order MNL serves as our benchmark, as opposed to the more sophisticated estimation in \cite{jagabathula2018partial,jagabathula2022personalized}, which considers $\ell_1$-regularization. The benchmark latent-class partial order MNL model is termed \textsc{Lc-Pomnl}.

\subsection{Synthetic Data} \label{subsec:syn}
The first stage in the data generation process is the creation of the ground truth. There are $n=10$ products. Each product is associated with a revenue uniformly, randomly, and independently chosen from $[1,5]$. We construct 2 types of ground truths: Rank List (or stochastic preference) and Mallows-based Rank List.

The Rank List model has a set of strict preference lists, each associated with a given probability. Each customer's strict preference is selected according to the given probability parameters. In general, Rank List is equivalent to RUM when the number of preferences is unlimited \cite{block1974random}. In the Rank List models that we construct, the first preference list is chosen with probability $0.1$ and has the no-purchase option as the first alternative.  Following \cite{berbeglia2022comparative}, we generate two ground truth models: Rank List with 10 lists (RL-10) and with 100 lists (RL-100). The remaining 0.9 probability is covered by 9 or 99 preference lists characterized as follows: the no-purchase option is not the top choice; each strict preference is drawn uniformly and independently at random from all possible preferences; each strict preference list $\ell$ occurs with a probability $0.9 \beta_\ell / (\sum_{\ell'} \beta_{\ell'})$ where $\beta_\ell \in (0,1)$ is drawn uniformly and independently at random for all list $\ell$.

To consider a scenario in which customers' strict preferences are correlated rather than independently drawn, we also generate Rank List ground truths based on the Mallows model \cite{mallows1957non}. The Mallows model is the most popular among distance-based ranking models, characterized by a central strict preference $\tau$ and a concentration parameter $\kappa \in (0,1)$. Here, $\tau: [n]_+ \to [n+1]$ where $j \in [n]_+$ is the $\tau(j)$-th most preferred product among $[n]_+$. The probability of a customer having a strict preference $\tau'$ is proportional to $\kappa^{d(\tau,\tau')}$, where $d(\tau, \tau'):=\sum_{i,j \in [n]_+:i < j} \mathbbm{1}[(\tau(i)-\tau(j))(\tau'(i)-\tau'(j)) < 0]$ is the Kendall-Tau distance between $\tau$ and $\tau'$. Namely, $d(\tau, \tau')$ counts the number of pairwise ordering disagreements between $\tau$ and $\tau'$. Intuitively, the Mallows model defines a set of similar preferences centered around a common permutation, where the deviation probability decreases exponentially in the number of disagreements. When $\kappa \to 1$, the preference distribution is uniform; when $\kappa \to 0$, the preference distribution is highly centered around $\tau$. Although the unconditional choice probability under Mallows can be efficiently computed according to \cite{desir2016assortment}, computing the conditional choice probability could be convoluted. One approach is to enumerate all possible complete linear extensions, but this yields $(n+1)!$ possibilities, and it is unclear how to use reduced linear extensions. Therefore, we construct Rank List ground truths that approximate the Mallows model, where the unconditional choice probability calculation is tractable when a few strict preferences are sampled.

We randomly and uniformly pick a central strict preference $\tau$ in which the no-purchase option is not the first alternative. 2000 strict preferences $\tau'$ are independently sampled proportionally to $\kappa^{d(\tau, \tau')}$. In the Mallows-based Rank List models that we construct, the first preference list is chosen with probability $0.2$ and has the no-purchase option as the first alternative. The remaining 0.8 is covered by the 2000 strict preference lists $\tau'$, each associated with probability $0.0004$. We merge the same preference lists by adding the probabilities. Two ground truth models are considered: Mallows-based Rank List with $\kappa=0.4$ (MRL-0.4) and with $\kappa=0.7$ (MRL-0.7).

In the second stage, we generate in-sample transaction data from the ground truth, following a setting similar to \cite{berbeglia2022comparative}. There are 2000 customers, each associated with an underlying strict preference list over $[n]_+$. There are 600 periods, each containing 10 transactions. Each period is associated with an offer set $S \subseteq [10]$ such that $3 \le |S| \le 6$. Each transaction contains a customer ID in $[2000]$, the offer set, and the product selected according to the customer's preference list. The customer of a transaction is chosen uniformly at random. On average, each customer has $600 \times 10 / 2000 = 3$ transactions. We note that if the same product $j$ is chosen from different offer sets, then these transactions are merged into one transaction where $j$ is chosen from the union of the offer sets, so the average value of $k_c \le 3$.

We generate 100 instances for each combination of ground truth and the number of in-sample customers used to estimate the choice model. The number of customers $C \in \{100, 250, 500, 1000, 2000\}$, and we use transactions associated with customers in $[C]$ for training, so that the nested structure ensures information gain as $C$ increases. We compare the MC choice models estimated by our EM algorithms \cus and \hyb with the benchmark models in Section \ref{subsec:ben}. We use customer IDs in $[C]$ to test \cus, which aims to maximize \eqref{eq:log-like-cus}. The estimated choice model is termed \cusmc. To test \hyb, we reveal half of the customer IDs in $[C]$. More specifically, only the odd customer IDs are revealed, while transactions associated with even customer IDs are added to the independent set. The estimated choice model is termed \hybmc. In practice, customers can be categorized as members or guests, and only the transactions made by members are recorded as customer-specific. For each performance measure, we take the average over the 100 instances. To test \lcmnl, we consider the number of segments $\ell=1,2,...,10$, and pick $\ell$ with the lowest average soft RMSE over the 100 instances. The number of segments for \lcpomnl is selected in the same way. This selection rule uses the known ground truth and is intended only to select a strong synthetic benchmark.\footnote{We use this approach to align with \cite{jagabathula2018partial,jagabathula2022personalized}. In applications without ground truth, $\ell$ should be selected using validation log-likelihood or another implementable model-selection criterion.} To avoid overfitting, we rule out the in-sample offer sets for the predictive RMSE measures in \eqref{eq:srmse} and \eqref{eq:srmse-con}. For computational reasons, when testing the conditional soft RMSE \eqref{eq:srmse-con} and the conditional revenue ratio \eqref{eq:rev-con}, we consider only $c \in [100]$, regardless of the number of in-sample customers $C$. This tests the impact of the richness of extra transaction data on the conditional measures of a fixed set of customers.

\subsubsection{Results} \label{subsubsec:syn-results}

Overall, we compare 7 different choice models, 2 from our framework and 5 benchmark models from the previous literature \cite{csimcsek2018expectation,train2008algorithms,jagabathula2018partial,jagabathula2022personalized}, on 4 different ground truths. For each ground truth, we have 5 different data sizes based on the number of customers, each with 100 instances. In total, 2000 instances are tested for each choice model estimation. The computationally intensive tasks were executed in parallel on a cluster at the University of Melbourne.

Before showing the experimental results, we present the metadata and characteristics of the in-sample transaction data from the 2000 customers for each ground truth in Table \ref{table:syn-md}. On average, customers drawn from the Rank List have more transactions after merging ($k_c$) than those drawn from the Mallows-based Rank List, resulting in a higher number of reduced linear extensions. For the Mallows-based Rank List, the lower the concentration parameter $\kappa$, the lower the number of distinct sampled preference lists, indicating that the preference is more concentrated around the central preference.

\begin{table}[ht]
\centering

\begin{minipage}{0.57\textwidth}
\centering
\begin{tabular}{lcccc}
\hline
ground truth & RL-10 & RL-100 & MRL-0.4 & MRL-0.7 \\
\hline
average $k_c$ & 1.839 & 1.834 & 1.691 & 1.671 \\
average \#RLE & 1.430 & 1.431 & 1.350 & 1.347 \\
\hline
\end{tabular}
\end{minipage}
\hfill
\begin{minipage}{0.4\textwidth}
\centering
\begin{tabular}{lcc}
\hline
ground truth & MRL-0.4 & MRL-0.7 \\
\hline
average \#lists & 1681.84 & 1999.37 \\
\hline
\end{tabular}
\end{minipage}

\caption{The average $k_c$ and the average \#RLE are the average (over 100 instances) of the average number of transactions after merging and reduced linear extensions (over 2000 in-sample customers), respectively; the average \#lists is the average number of distinct preference lists (including the first list that has no-purchase as the top choice, which occurs with probability 0.2) over 100 instances.}
\label{table:syn-md}
\end{table}

\begin{figure}[H]
    \centering
    \includegraphics[width=\textwidth]{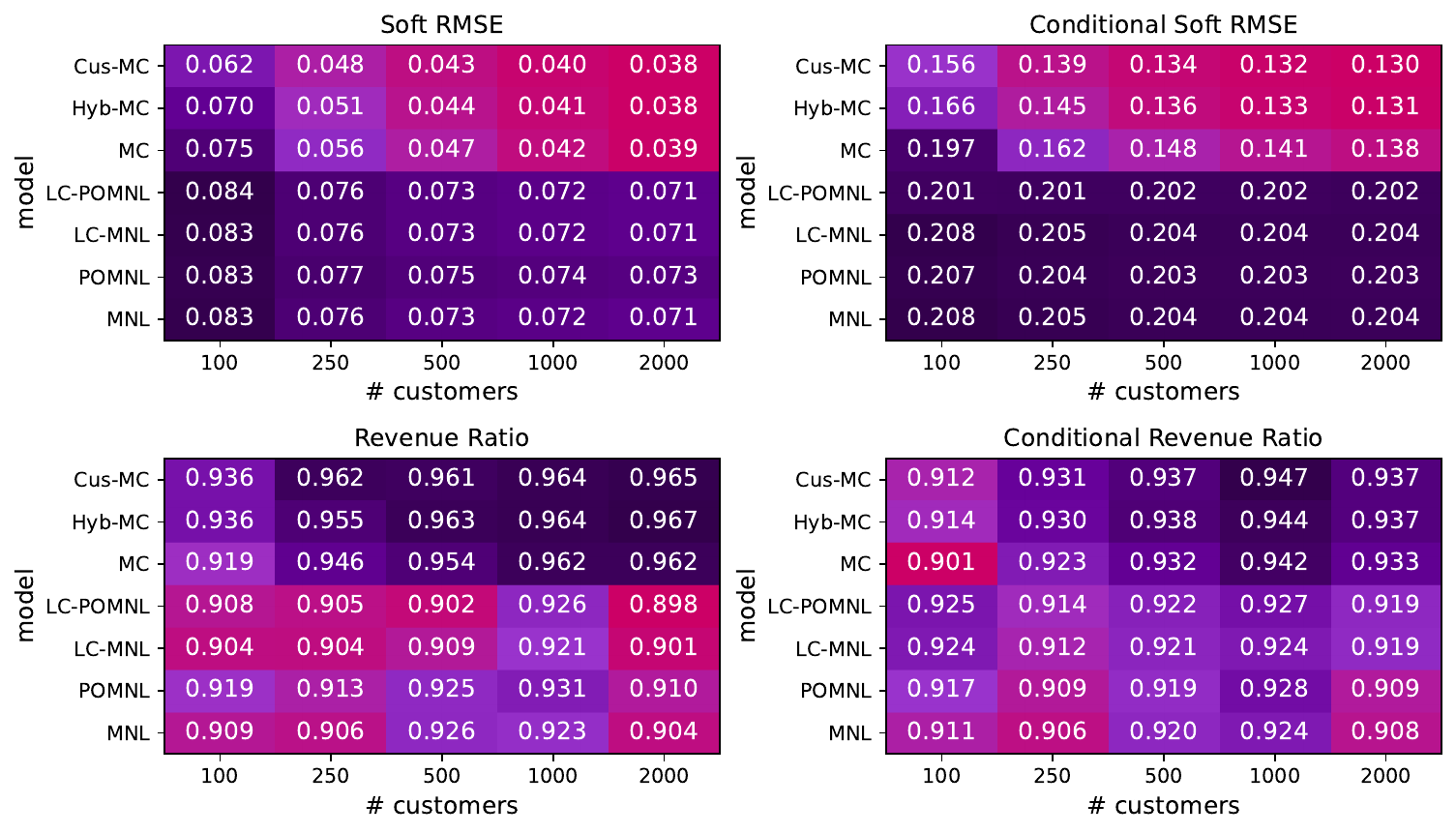}
	\caption{RL-10 Ground Truth} \label{fig:syn-rl-10}
\end{figure}

\begin{figure}[H]
    \centering
    \includegraphics[width=\textwidth]{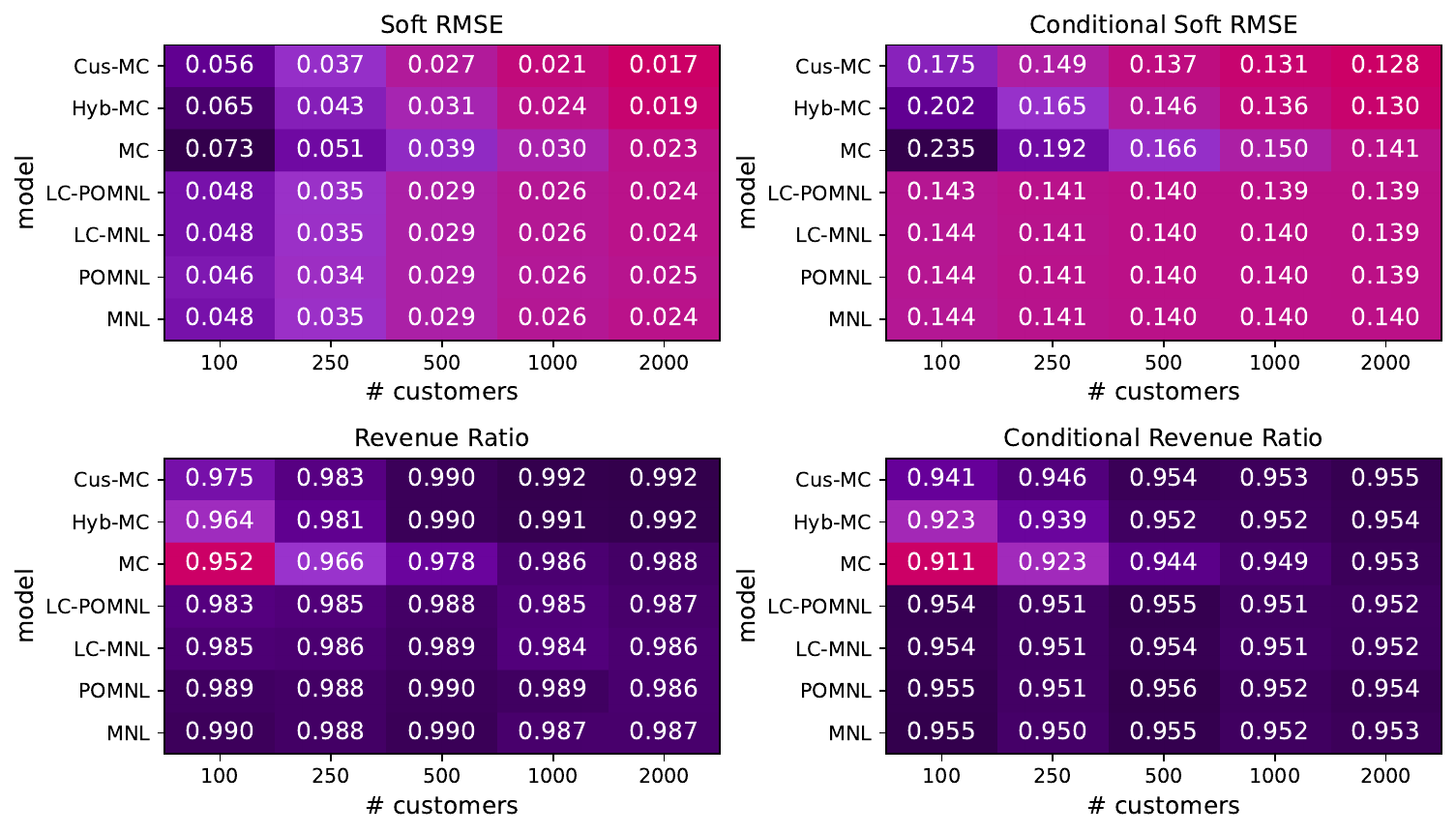}
	\caption{RL-100 Ground Truth} \label{fig:syn-rl-100}
\end{figure}

\begin{figure}[H]
    \centering
    \includegraphics[width=\textwidth]{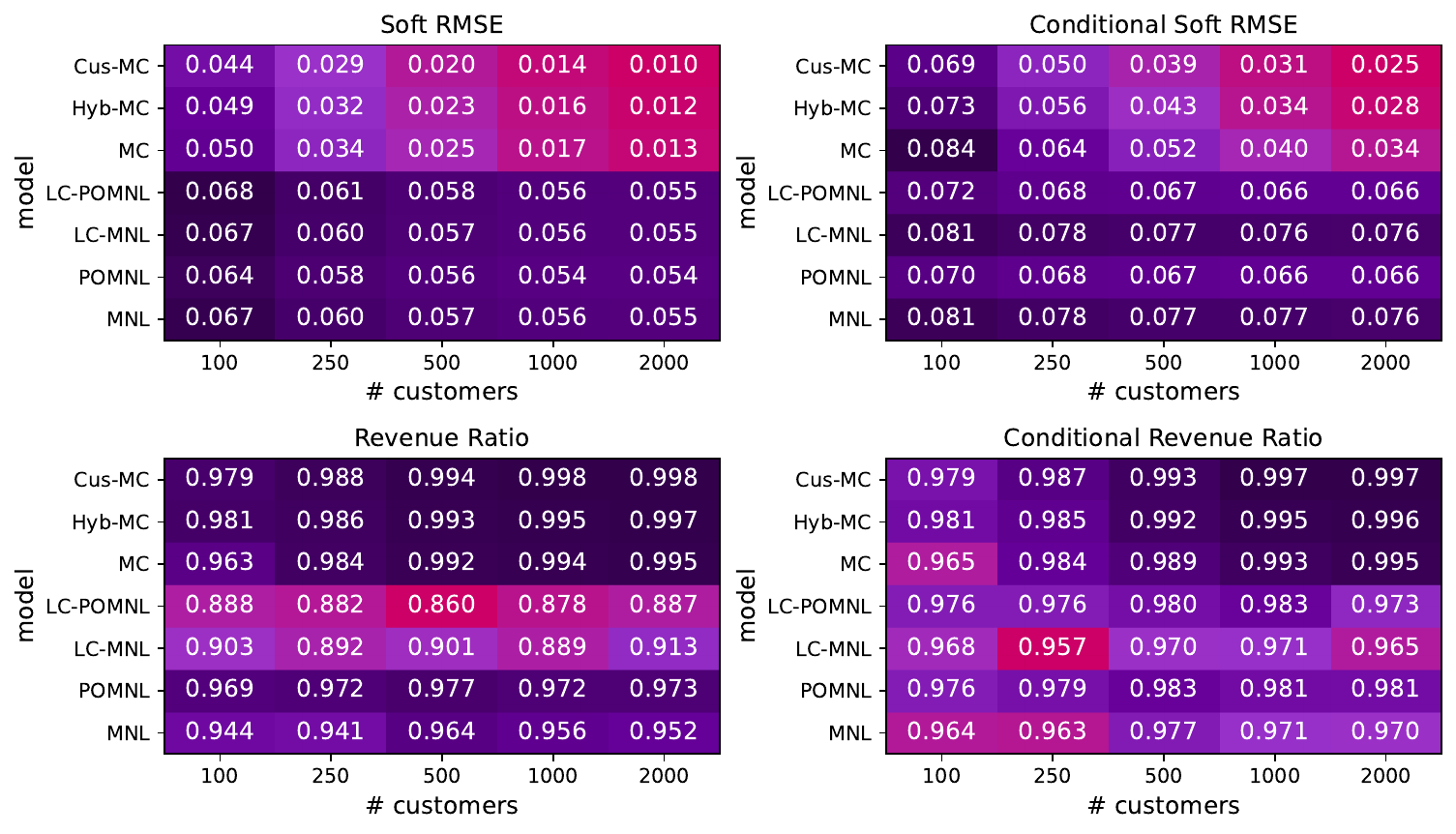}
	\caption{MRL-0.4 Ground Truth} \label{fig:syn-mrl-0.4}
\end{figure}

\begin{figure}[H]
    \centering
    \includegraphics[width=\textwidth]{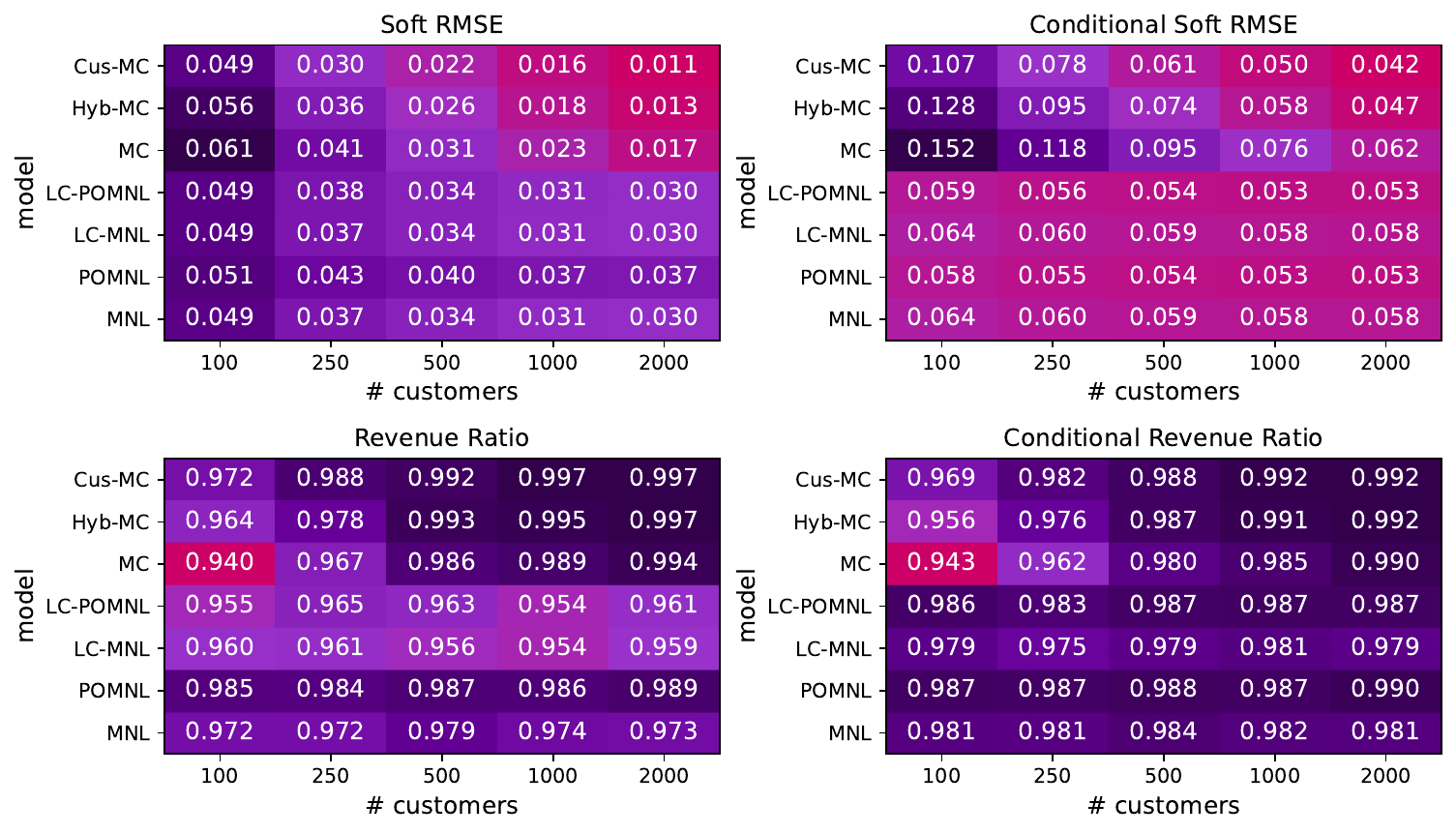}
	\caption{MRL-0.7 Ground Truth} \label{fig:syn-mrl-0.7}
\end{figure}

The experimental results are presented in Figures \ref{fig:syn-rl-10}, \ref{fig:syn-rl-100}, \ref{fig:syn-mrl-0.4}, and \ref{fig:syn-mrl-0.7}. We focus on two comparisons: first, the performance comparison between the MC-based choice models (\mc, \hybmc, and \cusmc) and the MNL-based choice models (\mnl, \pomnl, \lcmnl, and \lcpomnl); second, the improvement margin comparison between the traditional choice model estimation and the one with panel data. To ensure a more granular comparison, we also consider the impact of the training data size (number of customers) on the performance of each choice model estimation while making these comparisons.

In general, the MC-based choice models outperform the MNL-based choice models when the data size gets larger, or the ground truth distribution is more concentrated on a smaller number of preference lists (RL-10 and MRL-0.4), as opposed to spreading out more evenly. When preferences are more evenly distributed (RL-100 and MRL-0.7), MNL-based choice models outperform MC-based choice models when the data size is small in most cases. With large data sizes, MC-based choice models usually better capture the preference list distribution due to information richness.

In most cases, compared with MNL and LC-MNL, estimating an MC results in a greater improvement margin when customer-level information is incorporated into the choice model estimation, especially when the data size is small. Although this observation is based on synthetic experiments, it is theoretically supported in Section \ref{sec:val-cus}. When the ground truth is RL-10 or RL-100, \mc is the state-of-the-art in terms of the soft RMSE performance metric \cite{berbeglia2022comparative}. Our framework improves this result by exploiting individual information where transactions are correlated.

\subsection{Sushi Data} \label{subsec:sushi}

We assess the performance of estimation, prediction, and the impact on revenue using the sushi dataset, which contains preference lists of 5000 individuals \cite{kamishima2003nantonac}. In the commercial web survey used to collect preference data, each individual was first shown on the screen $n=10$ types of sushi. The individuals were asked to rank the sushi types according to their preferences. Each user provides one of the $10!$ permutations. Normalized prices are given from the data and used as the revenue vector. 

We follow a setting similar to \cite{berbeglia2022comparative}. Since the no-sushi option was not available in the survey, the ground truth model is generated by adding the \emph{no-sushi} list that ranks the no-sushi option as the top alternative. The no-sushi list is selected with probability half. With the remaining half, we select one list uniformly at random from the 5000 lists. We assume that an individual would purchase if one of the top $t$ is available. That is, the $(t+1)$-th choice is replaced with the no-sushi option. There are at most $10!/(10-t)!$ possible top $t$ preference lists. We merge the individuals with the same top $t$ lists and sum up the probabilities to construct the Rank List ground truth. We test $t=5$ and $t=7$ and follow the same data generation procedure, choice model estimation, and performance evaluation as in Section \ref{subsec:syn}. The metadata and characteristics of the in-sample transaction data for each ground truth are presented in Table \ref{table:sushi-md}.

\begin{table}[ht]
\centering

\begin{tabular}{lcc}
\hline
ground truth & top 5 & top 7 \\
\hline
average $k_c$ & 1.473 & 1.474 \\
average \#RLE & 1.288 & 1.318 \\
\#lists & 3274 & 4694 \\
\hline
\end{tabular}
\caption{The average $k_c$ and the average \#RLE are the average number of transactions after merging and reduced linear extensions (over 2000 in-sample customers), respectively; the \#lists is the number of distinct preference lists (including the no-sushi list, which occurs with probability 0.5) over 100 instances.}
\label{table:sushi-md}
\end{table}

\subsubsection{Results} \label{subsubsec:sushi-results}

The experimental results are presented in Figures \ref{fig:sushi-5} and \ref{fig:sushi-7}. In general, the MC-based choice models outperform the MNL-based choice models, except for the conditional revenue ratio under the top 5 ground truth. For \hybmc and \cusmc, this performance metric degrades as the data size gets larger. One potential reason is the lack of information to fit an MC when panel data is used.

\begin{figure}[H]
    \centering
    \includegraphics[width=\textwidth]{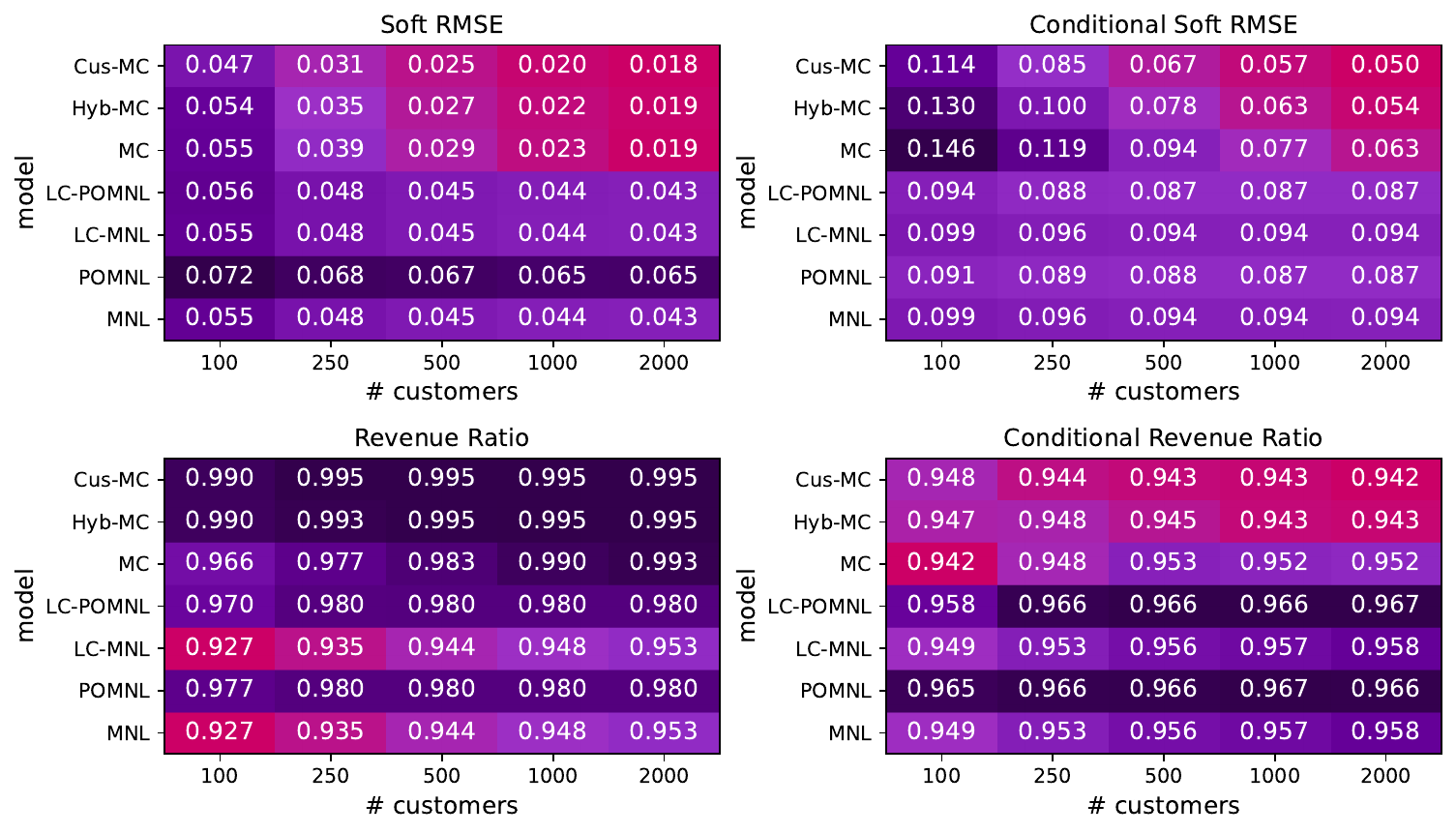}
	\caption{Top 5 Ground Truth} \label{fig:sushi-5}
\end{figure}
\begin{figure}[H]
    \centering
    \includegraphics[width=\textwidth]{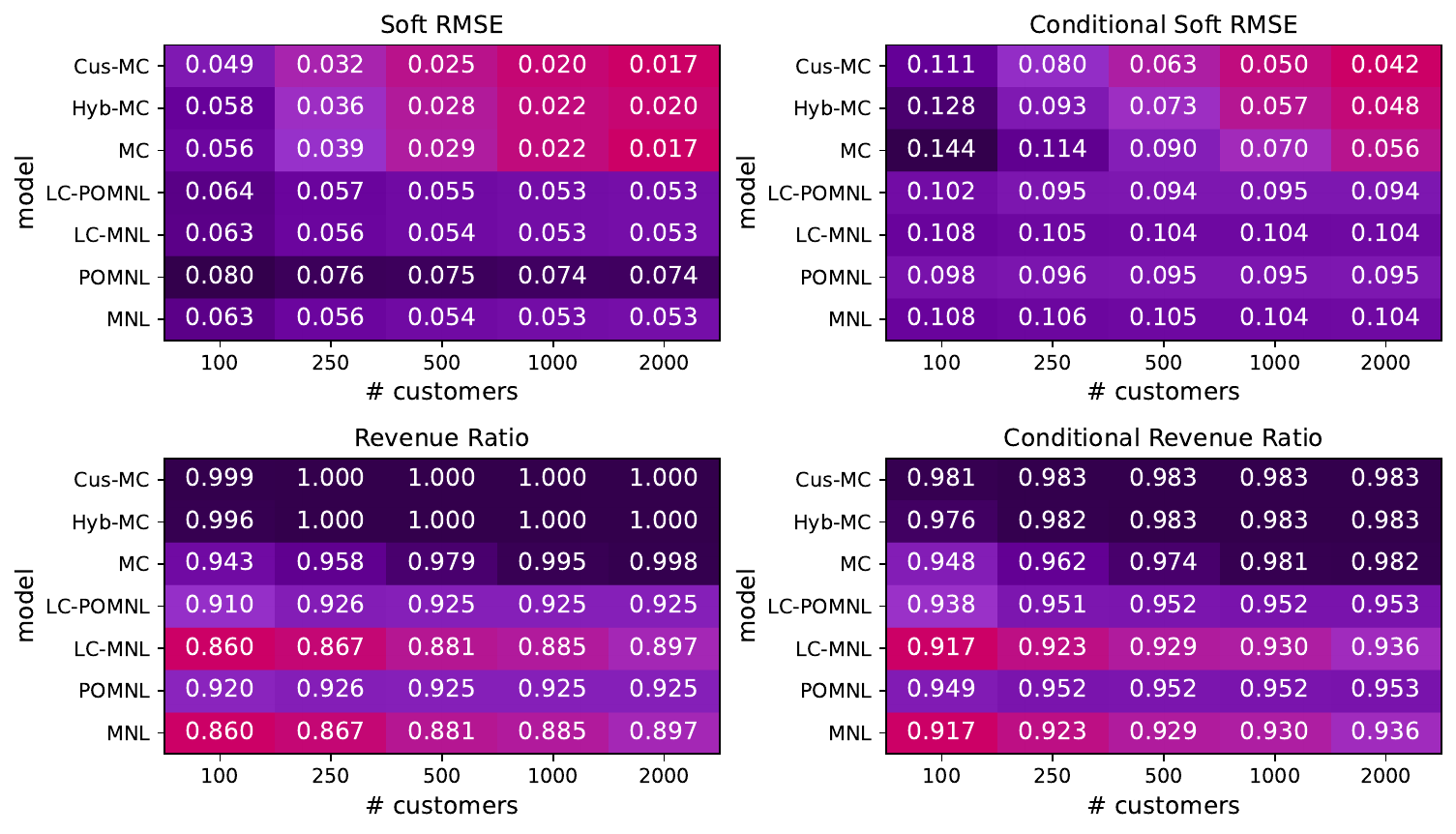}
	\caption{Top 7 Ground Truth} \label{fig:sushi-7}
\end{figure}

Compared with the top-7 truncation, the top-5 truncation might not provide sufficient information to incorporate customer-level information to fit an MC and improve the conditional revenue performance. Transaction information from other customers is not valuable for improving the conditional revenue ratio of the first 100 customers.

Compared with MNL and LC-MNL, estimating an MC results in a greater improvement margin when the choice model estimation uses panel data, especially for the predictive performance metrics, and when the data size is small. The revenue-based performance measures, on the other hand, do not fully reflect that using panel data is more beneficial for MC estimation than MNL and LC-MNL estimation. The factor for the revenue improvement margin for MNL and LC-MNL when panel data is used requires further investigation. In general, for the estimated choice models with panel data (\pomnl, \lcpomnl, \hybmc, and \cusmc), as the number of customers goes beyond 250, the improvement in the revenue-based metrics is marginal. Nevertheless, \hybmc and \cusmc obtain almost optimal expected revenue except for the conditional revenue under the top 5 ground truth.

\section{Conclusion} \label{sec:con}

In this work, we present EM algorithms \hyb and \cus for MC estimation using panel data. Our framework outperforms the traditional EM algorithm for MC estimation \cite{csimcsek2018expectation}, the traditional MNL-based estimation, and the partial-ordering-MNL-based estimation \cite{jagabathula2018partial,jagabathula2022personalized} on synthetic data and the semi-synthetic sushi data. We provide theoretical and empirical evidence that customer-level information, which captures correlations between transactions of the same customer, is more valuable for MC estimation.

Our primary future research direction is to test \hyb on real data, where customer preferences might be cyclic. This requires cycle deletion heuristics to extract a DAG that best describes the customer preference, as proposed in \cite{jagabathula2018partial,jagabathula2022personalized}. One potential advantage of \hyb is that arc-violating transactions, which might provide information for the unconditional population choice distribution, can be included in the independent set for training, instead of being discarded. Other research directions include choice prediction and estimation with panel data for other classes of RUMs, as well as investigating the value of customer-level information under different ground truth RUMs.

\bibliographystyle{alpha}
\bibliography{reference}

@article{jagabathula2022personalized,
  title={Personalized retail promotions through a directed acyclic graph--based representation of customer preferences},
  author={Jagabathula, Srikanth and Mitrofanov, Dmitry and Vulcano, Gustavo},
  journal={Operations Research},
  volume={70},
  number={2},
  pages={641--665},
  year={2022},
  publisher={Informs}
}

@article{blanchet2016markov,
  title={A \uppercase{M}arkov chain approximation to choice modeling},
  author={Blanchet, Jose and Gallego, Guillermo and Goyal, Vineet},
  journal={Operations Research},
  volume={64},
  number={4},
  pages={886--905},
  year={2016},
  publisher={INFORMS}
}

@article{berbeglia2022comparative,
  title={A comparative empirical study of discrete choice models in retail operations},
  author={Berbeglia, Gerardo and Garassino, Agust{\'\i}n and Vulcano, Gustavo},
  journal={Management Science},
  volume={68},
  number={6},
  pages={4005--4023},
  year={2022},
  publisher={Informs}
}

@inproceedings{kamishima2003nantonac,
  title={Nantonac collaborative filtering: recommendation based on order responses},
  author={Kamishima, Toshihiro},
  booktitle={Proceedings of the Ninth ACM SIGKDD International Conference on Knowledge Discovery and Data Mining},
  pages={583--588},
  year={2003}
}

@article{feldman2017revenue,
  title={Revenue management under the \uppercase{M}arkov chain choice model},
  author={Feldman, Jacob B. and Topaloglu, Huseyin},
  journal={Operations Research},
  volume={65},
  number={5},
  pages={1322--1342},
  year={2017},
  publisher={INFORMS}
}

@book{puterman2014markov,
  title={Markov decision processes: discrete stochastic dynamic programming},
  author={Puterman, Martin L.},
  year={2014},
  publisher={John Wiley \& Sons}
}

@article{gallego2004managing,
  title={Managing flexible products on a network},
  author={Gallego, Guillermo and Iyengar, Garud and Phillips, Robert and Dubey, Abhay},
  journal={Available at SSRN 3567371},
  year={2004}
}

@article{talluri2004revenue,
  title={Revenue management under a general discrete choice model of consumer behavior},
  author={Talluri, Kalyan and van Ryzin, Garrett},
  journal={Management Science},
  volume={50},
  number={1},
  pages={15--33},
  year={2004},
  publisher={INFORMS}
}

@article{beggs1981assessing,
  title={Assessing the potential demand for electric cars},
  author={Beggs, Steven and Cardell, Scott and Hausman, Jerry},
  journal={Journal of Econometrics},
  volume={17},
  number={1},
  pages={1--19},
  year={1981},
  publisher={Elsevier}
}

@article{chen2023assortment,
  title={Assortment optimization for the multinomial logit model with repeated customer interactions},
  author={Chen, Ningyuan and Gao, Pin and Wang, Chenhao and Wang, Yao},
  journal={Available at SSRN 4526247},
  year={2023}
}

@article{csimcsek2018expectation,
  title={An expectation-maximization algorithm to estimate the parameters of the \uppercase{M}arkov chain choice model},
  author={{\c{S}}im{\c{s}}ek, A. Serdar and Topaloglu, Huseyin},
  journal={Operations Research},
  volume={66},
  number={3},
  pages={748--760},
  year={2018},
  publisher={INFORMS}
}

@article{nettleton1999convergence,
  title={Convergence properties of the \uppercase{EM} algorithm in constrained parameter spaces},
  author={Nettleton, Dan},
  journal={Canadian Journal of Statistics},
  volume={27},
  number={3},
  pages={639--648},
  year={1999},
  publisher={Wiley Online Library}
}

@book{bishop2006pattern,
  title={Pattern recognition and machine learning},
  author={Bishop, Christopher M.},
  publisher={Springer},
  year={2006}
}

@inproceedings{brightwell1991counting,
  title={Counting linear extensions is \#\uppercase{P}-complete},
  author={Brightwell, Graham and Winkler, Peter},
  booktitle={Proceedings of the Twenty-third Annual ACM Symposium on Theory of Computing},
  pages={175--181},
  year={1991}
}

@incollection{block1974random,
  title={Random orderings and stochastic theories of responses (1960)},
  author={Block, Henry D.},
  booktitle={Economic Information, Decision, and Prediction: Selected Essays: Volume I Part I Economics of Decision},
  pages={172--217},
  year={1974},
  publisher={Springer}
}

@article{desir2020constrained,
  title={Constrained assortment optimization under the \uppercase{M}arkov chain--based choice model},
  author={D{\'e}sir, Antoine and Goyal, Vineet and Segev, Danny and Ye, Chun},
  journal={Management Science},
  volume={66},
  number={2},
  pages={698--721},
  year={2020},
  publisher={INFORMS}
}

@article{desir2022capacitated,
  title={Capacitated assortment optimization: Hardness and approximation},
  author={D{\'e}sir, Antoine and Goyal, Vineet and Zhang, Jiawei},
  journal={Operations Research},
  volume={70},
  number={2},
  pages={893--904},
  year={2022},
  publisher={INFORMS}
}

@article{bradley1952rank,
  title={Rank analysis of incomplete block designs: I. the method of paired comparisons},
  author={Bradley, Ralph Allan and Terry, Milton E.},
  journal={Biometrika},
  volume={39},
  number={3/4},
  pages={324--345},
  year={1952},
  publisher={JSTOR}
}

@article{mcfadden1974conditional,
  title={Conditional Logit Analysis of Qualitative Choice Behavior},
  author={McFadden, Daniel},
  journal={Frontiers in Econometrics},
  year={1974},
  publisher={Academic Press},
  pages={105--142}
}

@book{luce1959individual,
  title={Individual choice behavior: A Theoretical Analysis},
  author={Luce, R. Duncan},
  year={1959},
  publisher={Wiley New York}
}

@article{plackett1975analysis,
  title={The analysis of permutations},
  author={Plackett, Robin L.},
  journal={Journal of the Royal Statistical Society Series C: Applied Statistics},
  volume={24},
  number={2},
  pages={193--202},
  year={1975},
  publisher={Oxford University Press}
}

@article{jagabathula2018partial,
  title={A partial-order-based model to estimate individual preferences using panel data},
  author={Jagabathula, Srikanth and Vulcano, Gustavo},
  journal={Management Science},
  volume={64},
  number={4},
  pages={1609--1628},
  year={2018},
  publisher={Informs}
}

@article{berbeglia2016discrete,
  title={Discrete choice models based on random walks},
  author={Berbeglia, Gerardo},
  journal={Operations Research Letters},
  volume={44},
  number={2},
  pages={234--237},
  year={2016},
  publisher={Elsevier}
}

@article{dempster1977maximum,
  title={Maximum likelihood from incomplete data via the \uppercase{EM} algorithm},
  author={Dempster, Arthur P. and Laird, Nan M. and Rubin, Donald B.},
  journal={Journal of the Royal Statistical Society Series B: Methodological},
  volume={39},
  number={1},
  pages={1--22},
  year={1977},
  publisher={Wiley Online Library}
}

@article{train2008algorithms,
  title={\uppercase{EM} algorithms for nonparametric estimation of mixing distributions},
  author={Train, Kenneth E.},
  journal={Journal of Choice Modelling},
  volume={1},
  number={1},
  pages={40--69},
  year={2008},
  publisher={Elsevier}
}

@article{vulcano2012estimating,
  title={Estimating primary demand for substitutable products from sales transaction data},
  author={Vulcano, Gustavo and van Ryzin, Garrett and Ratliff, Richard},
  journal={Operations Research},
  volume={60},
  number={2},
  pages={313--334},
  year={2012},
  publisher={INFORMS}
}

@article{van2015market,
  title={A market discovery algorithm to estimate a general class of nonparametric choice models},
  author={van Ryzin, Garrett and Vulcano, Gustavo},
  journal={Management Science},
  volume={61},
  number={2},
  pages={281--300},
  year={2015},
  publisher={INFORMS}
}

@article{van2017expectation,
  title={An expectation-maximization method to estimate a rank-based choice model of demand},
  author={van Ryzin, Garrett and Vulcano, Gustavo},
  journal={Operations Research},
  volume={65},
  number={2},
  pages={396--407},
  year={2017},
  publisher={Informs}
}

@article{jagabathula2017nonparametric,
  title={A nonparametric joint assortment and price choice model},
  author={Jagabathula, Srikanth and Rusmevichientong, Paat},
  journal={Management Science},
  volume={63},
  number={9},
  pages={3128--3145},
  year={2017},
  publisher={INFORMS}
}

@article{anupindi1998estimation,
  title={Estimation of consumer demand with stock-out based substitution: An application to vending machine products},
  author={Anupindi, Ravi and Dada, Maqbool and Gupta, Sachin},
  journal={Marketing Science},
  volume={17},
  number={4},
  pages={406--423},
  year={1998},
  publisher={INFORMS}
}

@article{kok2007demand,
  title={Demand estimation and assortment optimization under substitution: Methodology and application},
  author={K{\"o}k, A G{\"u}rhan and Fisher, Marshall L.},
  journal={Operations Research},
  volume={55},
  number={6},
  pages={1001--1021},
  year={2007},
  publisher={INFORMS}
}

@article{conlon2013demand,
  title={Demand estimation under incomplete product availability},
  author={Conlon, Christopher T. and Mortimer, Julie Holland},
  journal={American Economic Journal: Microeconomics},
  volume={5},
  number={4},
  pages={1--30},
  year={2013},
  publisher={American Economic Association}
}

@article{stefanescu2009multivariate,
  title={Multivariate customer demand: modeling and estimation from censored sales},
  author={Stefanescu, Catalina},
  journal={Available at SSRN 1334353},
  year={2009}
}

@article{desir2024robust,
  title={Robust assortment optimization under the \uppercase{M}arkov chain choice model},
  author={D{\'e}sir, Antoine and Goyal, Vineet and Jiang, Bo and Xie, Tian and Zhang, Jiawei},
  journal={Operations Research},
  volume={72},
  number={4},
  pages={1595--1614},
  year={2024},
  publisher={INFORMS}
}

@article{li2025online,
  title={Online learning for constrained assortment optimization under \uppercase{M}arkov chain choice model},
  author={Li, Shukai and Luo, Qi and Huang, Zhiyuan and Shi, Cong},
  journal={Operations Research},
  volume={73},
  number={1},
  pages={109--138},
  year={2025},
  publisher={Informs}
}

@article{gupta2020parameter,
  title={Parameter identification in \uppercase{M}arkov chain choice models},
  author={Gupta, Arushi and Hsu, Daniel},
  journal={Theoretical Computer Science},
  volume={808},
  pages={99--107},
  year={2020},
  publisher={Elsevier}
}

@article{mallows1957non,
  title={Non-null ranking models. \uppercase{I}},
  author={Mallows, Colin L.},
  journal={Biometrika},
  volume={44},
  number={1/2},
  pages={114--130},
  year={1957},
  publisher={JSTOR}
}

@article{desir2016assortment,
  title={Assortment optimization under the \uppercase{M}allows model},
  author={D{\'e}sir, Antoine and Goyal, Vineet and Jagabathula, Srikanth and Segev, Danny},
  journal={Advances in Neural Information Processing Systems},
  volume={29},
  year={2016},
  pages={4707--4715}
}

@inproceedings{lovasz2006fast,
  title={Fast algorithms for log-concave functions: sampling, rounding, integration and optimization},
  author={Lov{\'a}sz, L{\'a}szl{\'o} and Vempala, Santosh},
  booktitle={2006 47th Annual IEEE Symposium on Foundations of Computer Science (FOCS'06)},
  pages={57--68},
  year={2006},
  organization={IEEE}
}

\newpage

\appendix

\section{Hardness and Computational Results} \label{sec:hnc}

In Section \ref{subsec:hard-fpras}, we establish \#P-hardness and randomized approximation schemes for computing the transaction probability under MNL. In Section \ref{subsec:ao}, we show hardness and computational results for the conditional assortment problems. As MNL is a special case of MC, the hardness results for MNL also hold for MC.

\subsection{Computing the Transaction Probability under MNL} \label{subsec:hard-fpras}

\subsubsection{\#P-hardness} \label{subsubsec:hard}

In this section, we show that computing the probability $\pi_\cD$ in \eqref{eq:def-cD} and the conditional probability $\pi_\cD(j,S)$ in \eqref{eq:cond-prob} are both \#P-hard under MNL. These results provide theoretical guarantees and necessitate heuristics to approximate the transaction probability \cite{jagabathula2018partial,jagabathula2022personalized}. In contrast to \cite{jagabathula2018partial,jagabathula2022personalized}, where the time complexity depends on the number of complete linear extensions, we consider reduced linear extensions, which could potentially improve computation time. We note that although these problems are hard in general, they are computationally tractable when $k! \in \poly(n)$.

We reduce the problem of counting the number of linear extensions \cite{brightwell1991counting} to the problem of computing the (conditional) transaction probability.

\begin{restatable}{theorem}{thmsp} \label{thm:sp}
Computing $\pi_\cD$ in \eqref{eq:def-cD} is \#P-hard under MNL.
\end{restatable}

\begin{proof}
    We reduce the problem of counting the number of linear extensions of a DAG, which is \#P-hard \cite{brightwell1991counting}, to the problem of computing $\pi_\cD$ under MNL.

    Suppose the DAG is $D = (V, A)$ with $V = [n]$, i.e., there are $n$ vertices labeled 1, 2, ..., $n$. We construct $\cD$ as follows. Let $[n]$ be the product set. Let the number of transactions $k=n$ and let $(j_\ell, S_\ell)$ be such that $j_\ell = \ell$ and $S_\ell = \{\ell\} \cup \{\ell' \mid \text{there is an $\ell \leadsto \ell'$ path in $D$}\}$. Let the attraction value of no-purchase $v_0 = n^n$.\footnote{Note that encoding $n^n$ only requires $n \log n$ bits, which is polynomial in the input size.} Let the attraction value of other products $v_j = 1$ for all $j \in [n]$.

    We claim that when $n \ge 5$, if $\pi_\cD$ can be computed in polynomial time, then the number of linear extensions of $D$ can be computed in polynomial time. We do not consider $n < 5$ since a brute-force algorithm can compute the number of linear extensions.

    Since each $i \in [n]$ appears to be chosen from $S_i$ in $\cD$, a reduced linear extension of $\cD$ ranks the products in $[n]$. Although the no-purchase option is not ranked in the reduced linear extension, it is the least preferred according to $\cD$. Furthermore, from the construction of $\cD$, $\sigma$ is a linear extension of $D$ if and only if $\sigma \sim_r \cD$. This defines a one-to-one mapping between a linear extension of $D$ and a reduced linear extension of $\cD$. Following \eqref{eq:prob-rle-mnl}, the probability of having a specific reduced linear extension $\sigma \sim_r \cD$ is
    \[\pi^\cD_\sigma = \prod_{i=1}^n \frac{1}{n^n + \left|\cup_{\ell=i}^n S_{j_{\sigma(\ell)}}\right|} \in \left[\frac{1}{\left(n^n+n\right)^n}, \frac{1}{\left(n^n+1\right)^n}\right].\]
    Let $c$ be the number of linear extensions of $D$, this implies that
    \[\pi_\cD \in \left[\frac{c}{\left(n^n+n\right)^n}, \frac{c}{\left(n^n+1\right)^n}\right].\]

    Let $I_c := \left[\frac{c}{\left(n^n+n\right)^n}, \frac{c}{\left(n^n+1\right)^n}\right]$ be an interval. Observe that $c \le n!$. We argue that $I_c \cap I_{c'} = \varnothing$ for any $c, c' \in [n!]$ with $c \ne c'$. That is, each interval $I_c$ is disjoint for $c \in [n!]$. Knowing $\pi_\cD$ implies that a binary search on $c$ can be used to retrieve the interval $I_c$ in polynomial time. Therefore, it suffices to show that
    \begin{equation} \label{ineq:thm-pi-cD}
        \frac{c}{(n^n+1)^n} < \frac{c+1}{(n^n+n)^n} \quad \forall c \in [n!].
    \end{equation}
    Equation \eqref{ineq:thm-pi-cD} is equivalent to
    \[\left(1+\frac{n-1}{1+n^n}\right)^n < 1 + \frac{1}{c} \quad \forall c \in [n!].\]
    We show the strictest inequality when $c = n!$:
    \begin{align*}
    \left(1+\frac{n-1}{1+n^n}\right)^n &\le \left(1+\frac{n}{n^n}\right)^n \le \left(1+\frac{1}{n^{n-1}}\right)^n = \sum_{q=0}^n {n \choose q} \frac{1}{n^{q(n-1)}}\\
    &= 1 + \frac{n}{n^{n-1}} + \frac{n(n-1)}{2!n^{2(n-1)}} + \frac{n(n-1)(n-2)}{3!n^{3(n-1)}} + ... + \frac{n}{n^{(n-1)(n-1)}} + \frac{1}{n^{n(n-1)}} \\
    &\le 1 + \frac{1}{n^{n-2}} + \frac{n^2}{n^{2(n-1)}} + \frac{n^3}{n^{3(n-1)}} + ... + \frac{n^{n-1}}{n^{(n-1)(n-1)}} + \frac{n^n}{n^{n(n-1)}} \\
    &\le \sum_{q=0}^\infty \left(\frac{1}{n^{n-2}}\right)^q = 1 + \frac{1}{n^{n-2}-1} < 1 + \frac{1}{n!}
\end{align*}
where the last inequality holds whenever $n \ge 5$.
\end{proof}

\begin{restatable}{theorem}{thmconsp} \label{thm:sp-con}
    Computing $\pi_\cD(j,S)$ in \eqref{eq:cond-prob} is \#P-hard under MNL.
\end{restatable}

\begin{proof}
    We show that under MNL, if $\pi_\cD(j,S)$ in \eqref{eq:cond-prob} can be computed in polynomial time, then $\pi_\cD$ in \eqref{eq:def-cD} can be computed in polynomial time.

    Given a transaction set $\cD$ under MNL, we sequentially add artificial transactions $(j_\ell, \{j_\ell, j'_\ell\})$ for $\ell=k+1, k+2, ..., k+p-1, k+p$ to $\cD$, until $\cD \cup \{(j_\ell, \{j_\ell, j'_\ell\}) \mid \ell=k+1, ..., k+p\}$ has a unique reduced linear extension, which is a strict ranking for $\{j \mid j \in [n]_+ \text{ such that } u_j > u_0\}$. Since the no-purchase option is always available, we do not know the ranking for products less preferred than no-purchase. Let $\cD_0 = \cD$ and $\cD_i = \cD_0 \cup \{(j_{k+\ell}, \{j_{k+\ell}, j'_{k+\ell}\}) \mid \ell \in [i]\}$ such that $\{\sigma \mid \sigma \sim_r \cD_i\} \subsetneq \{\sigma \mid \sigma \sim_r \cD_{i-1}\}$ for $i \in [p]$. Let $\sigma \sim_r \cD_p$ be the unique reduced linear extension. Using \eqref{eq:prob-rle-mnl}, $\pi^{\cD_p}_\sigma = \pi_{\cD_{p}}$ can be computed in polynomial time.

    From \eqref{eq:cond-prob}, we have that
    \begin{equation} \label{eq:thm-sp-con}
        \pi_{\cD_{i-1}}(j_{k+i}, \{j_{k+i}, j'_{k+i}\}) = \frac{\Pr[\cE_{\cD_i}]}{\Pr[\cE_{\cD_{i-1}}]} = \frac{\pi_{\cD_i}}{\pi_{\cD_{i-1}}} \quad \forall i \in [p].
    \end{equation}
    To compute $\pi_\cD = \pi_{\cD_0}$, observe that
    \[\pi_{\cD_{p}} = \pi_{\cD} \prod_{i=1}^{p} \frac{\pi_{\cD_i}}{\pi_{\cD_{i-1}}}.\]
    This implies that if \eqref{eq:thm-sp-con} can be computed in polynomial time, then $\pi_\cD$ can also be computed in polynomial time. From Theorem \ref{thm:sp}, computing $\pi_\cD(j,S)$ in \eqref{eq:cond-prob} is also \#P-hard under MNL.
\end{proof}

\subsubsection{Randomized Approximation Schemes}

To complement the \#P-hardness results in Section \ref{subsubsec:hard}, we present randomized approximation schemes for computing the probability $\pi_\cD$ in \eqref{eq:def-cD} and the conditional probability $\pi_\cD(j,S)$ in \eqref{eq:cond-prob} under MNL. When the attraction ratio is polynomially bounded, our results imply fully polynomial-time randomized approximation schemes (FPRAS) for these problems.

For notation brevity, let $v_{\max} := \max_{j \in [n]_+} v_j$ and $v_{\min} := \min_{j \in [n]_+} v_j$, and recall that $v_j > 0$ for all $j \in [n]_+$. Let the attraction ratio $R:= v_{\max} / v_{\min}$.

\begin{restatable}{theorem}{thmfpras} \label{thm:fpras} Under MNL, for any $\ep, \delta \in (0,1)$, there exists a randomized algorithm that returns an estimate $\hat{\pi}_{\cD}$ for
$\pi_\cD$ in \eqref{eq:def-cD}, satisfying $\Pr\left[\hat{\pi}_{\cD} \in [(1-\ep)\pi_{\cD}, (1+\ep)\pi_{\cD}]\right] \ge 1 - \delta,$ with running time polynomial in $n,R,1/\ep$, $\log(1/\delta)$, and the encoding length of $v$. If $R \in \poly(n)$, then $\pi_\cD$ admits an FPRAS.
\end{restatable}

\begin{proof}

If the directed graph induced by $D$ contains a cycle, then $\pi_\cD=0$, and the algorithm returns 0. Henceforth assume the graph is acyclic.

    The proof consists of three main steps. First, we define an exponential race that equivalently captures the MNL choice probability. More specifically, we introduce $n+1$ independent exponential random variables that equivalently capture the Gumbel random variables in MNL. Second, we consider the $(n+1)$-dimensional polytope induced by the partial ordering defined by $\cD$. The corresponding joint probability density function (PDF) is the product of independent exponential densities. We truncate the domain so that the excluded tail probability is negligible, and the truncated domain is well bounded. Third, since the polytope after truncation is well-structured and the joint PDF is log-concave, this allows us to adopt the randomized convex-body integration framework \cite{lovasz2006fast} to approximate the probability mass of the polytope. Combining these steps implies a randomized approximation scheme for $\pi_{\cD}$.

    Without loss of generality, we scale $v$ so that $v_j \in [1, R]$ for all $j \in [n]_+$.

\paragraph{\bf Exponential Race.} Let $T_j$ be independent random variables following the exponential distribution $\Exp(v_j)$. Here, $T_j$ can be regarded as the arrival time of element $j$.

Recall that under the exponential distribution, the support is in $x \in [0,\infty)$. For element $j$, the PDF is $f_j(x) = v_j \exp(-v_j x)$, and the cumulative distribution function (CDF) is $F_j(x) = 1- \exp(-v_j x)$. Under MNL, the utility of product $j$ is $u_j = \mu_j + \ep_j = \ln v_j + \ep_j$, where $\ep_j$ are independent random variables following the Gumbel distribution $\Gum$ with CDF $G_j(x) = \exp(-\exp(-x))$ for $x \in \R$. We show that the exponential arrival ordering is distributionally equivalent to the utility ordering under MNL.

\begin{claim} \label{cl:ex-gum}
    Let $\sigma:[n+1] \to [n]_+$ be a permutation, where $\sigma(r)$ is the $r$-th ranked element. Then 
    \[\Pr[u_{\sigma(1)} > u_{\sigma(2)} > ... > u_{\sigma(n+1)}] = \Pr[T_{\sigma(1)} < T_{\sigma(2)} < ... < T_{\sigma(n+1)}].\]
\end{claim}

\begin{proof}
    For each $j \in [n]_+$, let $\ep'_j := -\ln T_j - \mu_j$. We first show that $\ep'_j$ follows \Gum. For $x \in \R$,
    \begin{align*}
        \Pr[\ep'_j \le x] &= \Pr[-\ln T_j - \mu_j \le x] = \Pr[T_j \ge \exp(-x-\mu_j)] \\
        &= 1 - (1 - \exp(-v_j \exp(-x-\mu_j))) \\
        &= \exp\left(-\exp(\mu_j)\exp(-x-\mu_j) \right) = \exp(-\exp(-x)).
    \end{align*}
    This implies that $\ep_j$ and $\ep'_j$ are distributionally equivalent because both $\ep_j$ and $\ep'_j$ follow \Gum. Since the
\(T_j\)'s are independent, the \(\ep'_j\)'s are also independent, implying that $\ep_j$ and $\ep'_j$ are joint distributionally equivalent. Let $u'_j:= \mu_j + \ep'_j$. By construction, $u'_j:= \mu_j - \ln T_j - \mu_j = - \ln T_j$. Hence, for any $j,j'\in[n]_+$,
\[
u'_j > u'_{j'}
\iff
-\ln T_j>-\ln T_{j'}
\iff
T_j < T_{j'},
\]
because $-\ln y$ is strictly decreasing in $y \in (0,\infty)$.

Since $u_j$ and $u'_j$ are (joint) distributionally equivalent, for any ordering $\sigma$,
\[
\Pr\left[u_{\sigma(1)}>u_{\sigma(2)}>\cdots>u_{\sigma(n+1)}\right]
=
\Pr\left[u'_{\sigma(1)}>u'_{\sigma(2)}>\cdots>u'_{\sigma(n+1)}\right].
\]
Using $u'_j=-\log T_j$, the event on the right is exactly $T_{\sigma(1)}<T_{\sigma(2)}<\cdots<T_{\sigma(n+1)}$. Therefore,
\[
\Pr\left[u_{\sigma(1)}>u_{\sigma(2)}>...>u_{\sigma(n+1)}\right]
=
\Pr\left[T_{\sigma(1)}<T_{\sigma(2)}<...<T_{\sigma(n+1)}\right].
\]
\end{proof}

\paragraph{\bf Partial-ordering-based Polytope.}

From Claim \ref{cl:ex-gum}, we construct a polytope derived from the partial ordering based on $\cD$. Let $D=([n]_+, A)$ be a DAG induced by $\cD$. That is, if $(j,j') \in A$, then $u_j > u_{j'}$, or distributionally equivalently, $T_j < T_{j'}$. We have that $\pi_{\cD} = \Pr[T_j < T_{j'} \ \forall (j,j') \in A]$. The corresponding polytope is $K:=\{t \in \R_{\ge 0}^{n+1}: t_j \le t_{j'} \ \forall (j,j') \in A\}$. Let the joint PDF 
\begin{equation} \label{eq:ex-pdf}
    f(t):=\prod_{j=0}^n v_j \exp(-v_j t_j),
\end{equation}
then $\pi_{\cD} = \int_K f(t) dt$. Note that $f$ is log-concave in $t$ because $\ln f(t) = \sum_{j=0}^n \ln v_j - \sum_{j=0}^n v_j t_j$ is affine in $t$.

Because $D$ is a DAG, there exists at least one complete linear extension $\sigma \sim_c \cD$. Therefore
\[\pi_\cD \ge \pi_\sigma = \Pr\left[T_{\sigma(1)}<T_{\sigma(2)}<...<T_{\sigma(n+1)}\right] = \prod_{\ell=1}^{n+1}\frac{v_{\sigma(\ell)}}{\sum_{s=\ell}^{n+1}v_{\sigma(s)}} \ge \frac{1}{R^{n+1}(n+1)!}.
\]
Since $K$ is unbounded, we truncate $K$ to preserve sufficient probability mass for the analysis. Let 
\begin{equation} \label{eq:lbk}
    L:=\frac{1}{R^{n+1}(n+1)!}, B:=\ln \frac{4(n+1)}{\ep L}, \text{and } K_B:= K \cap [0,B]^{n+1}.
\end{equation}

Let $\tilde{\pi}_\cD:= \int_{K_B} f(t)dt$ be an estimated probability for $\pi_\cD$ after truncation. Since $v_j \ge 1$, we have that $\Pr[T_j > B] = \exp(-v_j B) \le \exp(-B)$ for $j \in [n]_+$. By the union bound, this implies
\[\Pr[\max_{j \in [n]_+} T_j > B] \le (n+1) \exp(-B) = \frac{\ep L}{4} \le \frac{\ep \pi_\cD}{4}\]
and $\pi_\cD - \tilde{\pi}_\cD \in [0,\ep \pi_\cD/4]$, or equivalently, 
\begin{equation} \label{eq:tilde-D}
    \tilde{\pi}_\cD \in [(1-\frac{\ep}{4})\pi_\cD, \pi_\cD].
\end{equation}
In the remainder, we approximate $\pi_\cD$ by approximating $\tilde{\pi}_\cD$, thereby completing the proof.

\paragraph{\bf Randomized Convex-body Integration.}

To approximate $\tilde{\pi}_\cD$, we use the theorem which follows from \cite{lovasz2006fast}.

\begin{theorem}[Adapted from \cite{lovasz2006fast}]
\label{thm:lv-logconcave}
Let \(K\subseteq \mathbb R^d\) be a full-dimensional compact convex body given by a
polynomial-time separation oracle. Suppose that we are given a point \(x\in K\) and
radii \(0<r_{\mathrm{in}}\le r_{\mathrm{out}}\) such that the Euclidean balls satisfy
$B_2(x,r_{\mathrm{in}})\subseteq K\subseteq B_2(x,r_{\mathrm{out}}).$
Let $\Gamma:=r_{\mathrm{out}}/r_{\mathrm{in}}$ and \(f:K\to \mathbb R_{> 0}\) be a log-concave function whose logarithm can be evaluated
to polynomial precision in polynomial time. Suppose that the logarithmic range of
\(f\) over \(K\) is bounded by \(\Delta\), i.e.,
\begin{equation} \label{eq:log-range}
    \ln\frac{\sup_{t\in K} f(t)}{\inf_{t\in K} f(t)}
\le \Delta .
\end{equation}
Then, for any \(\eta,\delta\in(0,1)\), there is a randomized algorithm that outputs
\(\hat{I}\) such that
\[
\Pr\left[
(1-\eta)\int_K f(t)\,dt
\le
\hat{I}
\le
(1+\eta)\int_K f(t)\,dt
\right]\ge 1-\delta .
\]
The running time is polynomial in $d, 1/\eta, \log(1/\delta), \Gamma$, $\Delta$, and the encoding length of the separation and evaluation oracles.
\end{theorem}

This theorem is a standard consequence of the randomized integration
framework for log-concave functions \cite{lovasz2006fast}. The aspect-ratio dependence is handled by the
standard rounding preprocessing: starting from a convex body with polynomially bounded
\(\Gamma\), one first applies an affine transformation that brings the support into
well-rounded position, and then applies the log-concave integration algorithm.

To use Theorem \ref{thm:lv-logconcave}, we show that the following requirements hold. First, there is an efficient separation oracle, i.e., it takes polynomial time to verify if any point $t \in \R^{n+1}$ belongs to $K_B$. Second, there exist $x \in K_B$, $r_\mathrm{in}$, and $r_\mathrm{out}$, such that $B_2(x,r_{\mathrm{in}})\subseteq K_B \subseteq B_2(x,r_{\mathrm{out}})$ and $\Gamma = 4(n+2)\sqrt{n+1}$. Finally, \eqref{eq:log-range} holds with $K=K_B$, $f$ defined in \eqref{eq:ex-pdf} is log-concave, and $\Delta = (n+1)RB \in \poly(n,R,1/\ep)$.

The convex body $K_B \subseteq \R^{n+1}$ is described by the linear inequalities $t_j \in [0,B]$ for all $j \in [n]_+$ and $t_j \le t_{j'}$ for all $(j,j') \in A$. This admits a polynomial-time separation oracle.

Let $\sigma \sim_c \cD$ and $x \in \R^{n+1}$ be such that $x_{\sigma(s)} = sB/(n+2)$ for $s \in [n+1]$. For every $(j,j') \in A$, we have $x_{j'} - x_j \ge B/(n+2)$. Let 
\[r_{\mathrm{in}}:=\frac{B}{4(n+2)} \text{ and } r_{\mathrm{out}}:=B\sqrt{n+1}.\]
If $\|y-x\|_2 \le r_{\mathrm{in}}$, then $y_j \in [0,B]$ for all $j \in [n]_+$ and \[y_{j'} - y_j \ge x_{j'} - x_j - 2 r_{\mathrm{in}} \ge \frac{B}{2(n+2)} > 0.\]
This implies $B_2(x,r_{\mathrm{in}})\subseteq K_B$. For any $y \in K_B \subseteq [0,B]^{n+1}$, $\|y-x\|_2 \le \sqrt{\sum_{j=1}^{n+1} B^2} = B \sqrt{n+1}$, implying that $K_B \subseteq B_2(x,r_{\mathrm{out}})$. Consequently, $\Gamma = r_{\mathrm{out}}/r_{\mathrm{in}} = 4(n+2)\sqrt{n+1}$.

Since $K_B \subseteq [0,B]^{n+1}$ and $v_j \in [1,R]$ for all $j \in [n]_+$, using \eqref{eq:ex-pdf} implies that 
\begin{equation*}
    \ln\frac{\sup_{t\in K_B} f(t)}{\inf_{t\in K_B} f(t)} \le \ln \frac{\prod_{j=0}^n v_j}{\prod_{j=0}^n v_j \exp(-v_j B)} = B\sum_{j=0}^nv_j \le (n+1)BR.
\end{equation*}
Moreover, from \eqref{eq:lbk}, we have that $B=O(n \log R + n \log n + \log (1/\ep))$.

By setting $\eta = \ep/4$, Theorem \ref{thm:lv-logconcave} implies a randomized approximation scheme that outputs an estimate $\hat{\pi}_\cD$ for $\tilde{\pi}_\cD$ such that
\begin{equation} \label{eq:apx-tilde}
    \Pr\left[
(1-\frac{\ep}{4})\tilde{\pi}_\cD
\le
\hat{\pi}_\cD
\le
(1+\frac{\ep}{4})\tilde{\pi}_\cD
\right]\ge 1-\delta
\end{equation}
with running time polynomial in $n$, $R$, $1/\ep$, and $\log(1/\delta)$. Combining \eqref{eq:tilde-D} and \eqref{eq:apx-tilde}, on the success event of the integration algorithm, we have that
\[\hat{\pi}_\cD \ge (1-\frac{\ep}{4})\tilde{\pi}_\cD \ge (1-\frac{\ep}{4})^2 \pi_\cD \ge (1-\ep) \pi_\cD
\]
where the last inequality follows since $\ep \in (0,1)$. Similarly,
\[\hat{\pi}_\cD \le (1+\frac{\ep}{4})\tilde{\pi}_\cD \le (1+\ep)\tilde{\pi}_\cD \le (1+\ep)\pi_\cD.
\]
Therefore, $\Pr\left[\hat{\pi}_{\cD} \in [(1-\ep)\pi_{\cD}, (1+\ep)\pi_{\cD}]\right] \ge 1 - \delta$, which proves the existence of the randomized approximation scheme.
\end{proof}

\begin{restatable}{theorem}{thmfprascon} \label{thm:fpras-con} Under MNL, for any $\ep, \delta \in (0,1)$ and $\cD$ where the induced directed graph is acyclic, there exists a randomized algorithm that returns an estimate $\hat{\pi}_{\cD}(j,S)$ for
$\pi_\cD(j,S)$ in \eqref{eq:cond-prob}, satisfying 
\[\Pr\left[\hat{\pi}_{\cD}(j,S) \in [(1-\ep)\pi_{\cD}(j,S), (1+\ep)\pi_{\cD}(j,S)]\right] \ge 1 - \delta,\] with running time polynomial in $n,R,1/\ep$, $\log(1/\delta)$, and the encoding length of $v$. If $R \in \poly(n)$, then $\pi_\cD(j,S)$ admits an FPRAS.
\end{restatable}

\begin{proof}

If $j \notin S_+$, the algorithm returns \(0\). Henceforth, we consider the case \(j\in S_+\).

Let $\eta = \ep/4$, $\xi = \delta/4$, and $\pi_\cD(j,S) = \pi_{\cD'} / \pi_\cD$ with $\cD' = \cD \cup \{(j,S)\}$. If the directed graph induced by $\cD'$ has a cycle, return 0. Otherwise, by Theorem \ref{thm:fpras}, there exists a randomized algorithm that returns estimates $\hat{\pi}_{\cD}$ for $\pi_\cD$ and $\hat{\pi}_{\cD'}$ for $\pi_{\cD'}$, satisfying 
\[\Pr\left[\hat{\pi}_{\cD} \in [(1-\eta)\pi_{\cD}, (1+\eta)\pi_{\cD}]\right] \ge 1 - \xi \text{ and } \Pr\left[\hat{\pi}_{\cD'} \in [(1-\eta)\pi_{\cD'}, (1+\eta)\pi_{\cD'}]\right] \ge 1 - \xi,\] with running time polynomial in $n,R,1/\eta$, and $\log(1/\delta)$.

The probability that either $\hat{\pi}_{\cD} \in [(1-\eta)\pi_{\cD}, (1+\eta)\pi_{\cD}]$ or $\hat{\pi}_{\cD'} \in [(1-\eta)\pi_{\cD'}, (1+\eta)\pi_{\cD'}]$ fails is at most $2 \xi = \delta/2 \le \delta$, which implies that 
\[\Pr\left[\hat{\pi}_{\cD} \in [(1-\eta)\pi_{\cD}, (1+\eta)\pi_{\cD}] \text{ and } \hat{\pi}_{\cD'} \in [(1-\eta)\pi_{\cD'}, (1+\eta)\pi_{\cD'}]\right] \ge 1 - \delta.\]
Let $\hat{\pi}_{\cD}(j,S) = \hat{\pi}_{\cD'} / \hat{\pi}_{\cD}$. When both events are successful, we have that
\[\hat{\pi}_{\cD}(j,S) = \frac{\hat{\pi}_{\cD'}}{\hat{\pi}_{\cD}} \ge \frac{(1-\eta)\pi_{\cD'}}{(1+\eta)\pi_{\cD}} = \frac{(1-\ep/4)\pi_{\cD'}}{(1+\ep/4)\pi_{\cD}} \ge \frac{(1-\ep)\pi_{\cD'}}{\pi_{\cD}} = (1-\ep)\pi_{\cD}(j,S)\]
where the last inequality holds because $(1+\ep/4)(1-\ep) = 1 - 3\ep/4 - \ep^2/4 \ge 1 - \ep$ when $\ep \in (0,1)$. Similarly,
\[\hat{\pi}_{\cD}(j,S) = \frac{\hat{\pi}_{\cD'}}{\hat{\pi}_{\cD}} \le \frac{(1+\eta)\pi_{\cD'}}{(1-\eta)\pi_{\cD}} = \frac{(1+\ep/4)\pi_{\cD'}}{(1-\ep/4)\pi_{\cD}} \le \frac{(1+\ep)\pi_{\cD'}}{\pi_{\cD}} = (1+\ep)\pi_{\cD}(j,S)\]
where the last inequality holds because $(1-\ep/4)(1+\ep) = 1 + 3\ep/4 - \ep^2/4 \le 1 + \ep$ when $\ep \in (0,1)$.
\end{proof}

\subsection{Hardness and Computability of Conditional Assortment Optimization} \label{subsec:ao}

MC and MNL are widely used in assortment planning due to their computational tractability for \eqref{opt:tao}. However, conditioning on historical transactions significantly alters the problem structure, thus rendering them less tractable even under MNL. We show that \eqref{opt:cao} is NP-hard under MNL conditional on one transaction. This proof is adapted from \cite{chen2023assortment}, which shows the hardness of a multi-stage conditional assortment problem.\footnote{We note that it is possible to design an FPTAS for \eqref{opt:cao} under MNL based on \cite{chen2023assortment} when $k$ is constant.} On the positive side, we show that \eqref{opt:tcao} under MC is computationally tractable when $k$ is small, using an LP-based approach similar to LP \eqref{lp:mc-tao}.

\begin{restatable}{theorem}{thmcaomnl} \label{thm:caomnl}
    \eqref{opt:cao} is NP-hard under MNL, even when $k=1$.
\end{restatable}

\begin{proof}
We follow the ideas in \cite{chen2023assortment}. The proof consists of two main parts. First, we provide closed-form expressions for the conditional probability $\pi_\cD(j,S)$ when $\cD=\{j_1, S_1\} = \{(i, \{i\})\}$ for an $i \in [n]$. Second, we reduce an NP-complete problem to \eqref{opt:cao}, where the analysis involves the closed-form expressions.

Suppose $\cD=\{(j_1, S_1)\} = \{(i, \{i\})\}$ for an $i \in [n]$. From \eqref{eq:cond-prob} and \eqref{eq:prob-rle-mnl},\footnote{Or alternatively, \cite{chen2023assortment}.} we have that if $i \in S$, then
\begin{equation} \label{eq:pi-cD-in}
    \pi_\cD(j,S) = 
    \begin{cases}
        \mbox{\large $\frac{v_0+v_i}{V(S_+)}$} & \text{if } j=i, \\
        0 & \text{if } j=0, \\
        \mbox{\large $\frac{v_j}{V(S_+)}$} & \text{if } j \in S \setminus \{i\};\\
    \end{cases}
\end{equation}
if $i \notin S$, then
\begin{equation} \label{eq:pi-cD-notin}
    \pi_\cD(j,S) = 
    \begin{cases}
        \mbox{\large $\frac{v_0}{V(S_+)}\frac{v_0+v_i}{v_i+V(S_+)}$} & \text{if } j=0, \\
        \mbox{\large $\frac{v_j}{v_i+V(S_+)} + \frac{v_j (v_0+v_i)}{V(S_+)(v_i+V(S_+))}$} & \text{if } j \in S.\\
    \end{cases}
\end{equation}

Now we reduce the partition problem, which is NP-complete, to \eqref{opt:cao}. In the partition problem, we have $c_i > 0$ for $i \in [m]$ with $\sum_{i \in [m]} c_i = 2T$ for an integer $m > 0$. The goal is to find a subset of numbers $P \subseteq [m]$ such that $\sum_{i \in P}c_i = \sum_{i \in [m] \setminus P}c_i = T$.

Given an instance of the partition problem, we construct an instance of \eqref{opt:cao} under MNL with $k=1$ as follows. We have $n=m+2$ products with the following revenues and attraction values:
\begin{equation} \label{rv}
    r_i = 
    \begin{cases}
        \frac{69}{37} & \text{if } i \in [m], \\
        3 & \text{if } i=m+1, \\
        1 & \text{if } i=m+2, \\
    \end{cases}
    \text{ and }
    v_i = 
    \begin{cases}
        \frac{c_i}{T} & \text{if } i \in [m], \\
        1 & \text{if } i=0, m+1, m+2.\\
    \end{cases}
\end{equation}
Let $\cD = \{(j_1,S_1)\} = \{(m+2, \{m+2\})\}$. Let $S^* \subseteq [m+2]$ be an optimal offer set for \eqref{opt:cao}. We show that $m+2 \notin S^*$ and $m+1 \in S^*$.

Suppose $m+2 \in S^*$, then from \eqref{eq:pi-cD-in}, \eqref{opt:cao} is equivalent to \eqref{opt:tao} with input parameters $r'$ and $v'$ described as follows. Since $m+2$ is already selected from $\{0, m+2\}$. We know that offering $m+2$ ensures a revenue of $r_{m+2}$. Therefore, we set the revenues $r'_i = r_i - r_{m+2}$ as the shifted revenues and $v'_i = v_i$ for $i \in [m+1]$, and set $m+2$ as the no-purchase option with an attraction value of $v_0 + v_{m+2}$. Then we consider revenue-ordered assortments and have that $S = [n] = [m+2]$ is optimal for \eqref{opt:cao}, providing an expected revenue of
\begin{equation} \label{rev_n2}
    \frac{3\cdot v_{m+1} + 69/37\cdot 2 + 1 \cdot (v_0 + v_{m+2})}{v_0+\sum_{i \in [m+1]} v_i+v_{m+2}} = \frac{3 + 118/37 + 2}{5} = \frac{323}{185} < 2.
\end{equation}

Suppose $m+2 \notin S^*$, then from \eqref{eq:pi-cD-notin}, for any $S \subseteq [m+1]$, the revenue obtained by showing $S$ is
\begin{align*}
    &\quad\sum_{j \in S} \left(\frac{1}{v_i+V(S_+)} + \frac{v_0+v_i}{V(S_+)(v_i+V(S_+))}\right)r_j v_j = \sum_{j \in S} \left(\frac{1}{2+V(S)} + \frac{2}{(1+V(S))(2+V(S))}\right)r_j v_j \\
    &= \sum_{j \in S} \left(\frac{2 + (1+V(S))}{(1+V(S))(2+V(S))}\right)r_j v_j = \sum_{j \in S} \left(\frac{2(2+V(S)) - (1+V(S))}{(1+V(S))(2+V(S))}\right)r_j v_j \\
    &= \sum_{j \in S}2\left(\frac{1}{1+V(S)} - \frac{1}{2(2+V(S))}\right)r_j v_j.
\end{align*}

Therefore, $S^*$ is a solution of
\begin{equation} \label{opt:thm-cao}
    \max_{S \subseteq [m+1]}\sum_{j \in S}2\left(\frac{1}{1+V(S)} - \frac{1}{2(2+V(S))}\right)r_j v_j.
\end{equation}

The following lemma is from \cite{chen2023assortment}.

\begin{lemma}[\cite{chen2023assortment}]
Given $r$ and $v$ from \eqref{rv}, the optimal solution $S^*$ for \eqref{opt:thm-cao} is such that $S^* = S' \cup \{m+1\}$ for some $S' \subseteq [m]$. In addition, the expected revenue of $S^*$ is $75/37 > 2$ if and only if there exists $S' \subseteq [m]$ such that $V(S') = 1$.
\end{lemma}

Combining the above lemma and \eqref{rev_n2} implies that $m+2 \notin S^*$ and $m+1 \in S^*$. The partition problem has a solution if and only if there exists an assortment with expected revenue $75/37$. Suppose that there is a solution to the partition problem, then there exists $S' \subseteq [m]$ such that $\sum_{i \in S'} \frac{c_i}{T} = 1$, so the expected revenue is $75/37$. On the other hand, if there is no solution to the partition problem, then for any $S' \subseteq [m]$, the expected revenue using the assortment $S' \cup \{m+1\}$ is strictly less than $75/37$.
\end{proof}

\begin{theorem} \label{thm:tcaomc}
    There exists an algorithm for \eqref{opt:tcao} under MC that takes time polynomial in $k!$, $n$, and the encoding length of $\lambda,\rho,r$, and $\cD$. If $k! \in \poly(n)$, then \eqref{opt:tcao} under MC is polynomial-time solvable.
\end{theorem}

\begin{proof}
    Given $\cD=\{(j_\ell, S_\ell) \mid \ell \in [k]\}$, we recall the closed-form expression for $\pi_\cD(j,S)$ when $S \sim_t \cD$. Then we adapt LP \eqref{lp:mc-tao} to find the optimal assortment $S^*$ for \eqref{opt:tcao}.

    Recall that $A^\sigma_s := \overline{\cup_{\ell=s}^{k} S_{\sigma(\ell)}}$, and from \eqref{eq:p-sigma} that for $\sigma \sim_r \cD$,
    \begin{equation*}
        p_\cD(\sigma):=\prod_{s = 1}^{k-1}\Pr[j_{\sigma(s)} {\overset{A^\sigma_{s+1}}{\leadsto}} j_{\sigma(s+1)}]
    \end{equation*}
    denotes the probability of having a random walk in the form of $j_{\sigma(1)}\overset{A^\sigma_2}{\leadsto} j_{\sigma(2)} \overset{A^\sigma_3}{\leadsto} ...  \overset{A^\sigma_k}{\leadsto} j_{\sigma(k)}$ after the first visit to $j_\sigma(1)$. Recall that $U := \cup_{\ell \in [k]}S_{\ell}$. From \eqref{eq:prob-rle-mc} and \eqref{eq:prob-rle}, we have that \begin{equation} \label{prob:pi-cD-mc}
    \pi_\cD = \sum_{\ell \in [k]}\pi(j_\ell, U) \left(\sum_{\sigma \sim_r \cD: \sigma(1)=\ell} p_\cD(\sigma)\right)
    \end{equation}
    where the term $\pi(j_\ell, U)$ is the probability of having a prefix random walk until the first visit of $j_\ell$ for some $\ell \in [k]$, and the term $\sum_{\sigma \sim_r \cD: \sigma(1)=j_\ell} p_\cD(\sigma)$ considers all possible suffix random walks.

We consider three cases to compute $\Pr[\cE_\cD \cap \cE(j,S)]$ for any $S \sim_t \cD$. First, if $j \in S \setminus U$, then the customer's first visit to $S \cup U$ enters at $j$. After the first visit to $j$, the customer continues the random walk from $j$ to $j_\ell$, through states in $\overline{U}$. The remaining suffix, starting from the first visit of $j_{\sigma(1)}$, is consistent with some $\sigma$, where $\sigma(1) = \ell$ and $\sigma \sim_r \cD$, with probability $p_\cD(\sigma)$. Therefore,
\begin{equation} \label{eq:prob-in}
    \Pr[\cE_\cD \cap \cE(j,S)] = \pi(j,S \cup U) \left(\sum_{\sigma \sim_r \cD} \Pr[j \overset{\overline{U}}{\leadsto} j_{\sigma(1)}] p_\cD(\sigma)\right).
\end{equation}    
Second, let $J = \{j_\ell \mid \ell \in [k]\}$. If $j \in J$, i.e., $j = j_\ell$ for some $\ell \in [k]$, then the customer's first visit to $S \cup U$ enters at $j = j_\ell$. The remaining suffix, starting from the first visit to $j_\ell$, is consistent with some $\sigma$, where $\sigma(1) = \ell$ and $\sigma \sim_r \cD$, with probability $p_\cD(\sigma)$. Therefore,
\begin{equation} \label{eq:prob-notin}
    \Pr[\cE_\cD \cap \cE(j,S)] = \pi(j,S \cup U) \left(\sum_{\sigma \sim_r \cD: \sigma(1)=\ell} p_\cD(\sigma)\right).
\end{equation}
Finally, if $j \in S \cap (U \setminus J)$, then $\Pr[\cE_\cD \cap \cE(j,S)] = 0$ since there exists a more preferred product in $S \cap J$.

From \eqref{eq:cond-prob}, \eqref{opt:tcao} under MC is
\begin{equation} \label{opt:tcao-mc}
    \max_{S \sim_t \cD} \sum_{j \in S} \frac{\Pr[\cE_\cD \cap \cE(j,S)]}{\pi_\cD} r_j.
\end{equation}
Observe that in \eqref{prob:pi-cD-mc}, \eqref{eq:prob-in}, and \eqref{eq:prob-notin}, the terms $\pi_\cD$, $\sum_{\sigma \sim_r \cD} \Pr[j \overset{\overline{U}}{\leadsto} j_{\sigma(1)}] p_\cD(\sigma)$, and $\sum_{\sigma \sim_r \cD: \sigma(1)=\ell} p_\cD(\sigma)$ are irrelevant to $S$, only depend on $\cD$, and can be computed in time polynomial in $k!$ and $n$. Let the \emph{scaled} revenue of product $j$ be
\[
r^\cD_j := 
\begin{cases}
\mbox{\large $\frac{\sum_{\sigma \sim_r \cD} \Pr[j \overset{\overline{U}}{\leadsto} j_{\sigma(1)}] p_\cD(\sigma)}{\pi_\cD}$} r_j & \text{if } j \in \overline{U},\\
\mbox{\large $\frac{\sum_{\sigma \sim_r \cD: \sigma(1)=\ell} p_\cD(\sigma)}{\pi_\cD}$}r_j & \text{if } j = j_\ell \text{ for some } \ell \in [k], \\
0 & \text{otherwise}.
\end{cases}
\]
Then \eqref{opt:tcao-mc} can be written as
\begin{equation} \label{opt:tcao-mc-2}
    \max_{S \sim_t \cD} \sum_{j \in S} \pi(j,S \cup U) r^\cD_j.
\end{equation}

Observe that \eqref{opt:tcao-mc-2} is in a form similar to \eqref{opt:tao}. The main tweak while adapting LP \eqref{lp:mc-tao} is that we can regard the transparent conditional assortment problem as if all the products in $U$ are included in $S$ without loss of optimality, because the chosen product $j$ is either from $S \setminus U$ or $J$, regardless of which products in $U \setminus J$ are included in $S$. Therefore, we have the following LP formulation for \eqref{opt:tcao-mc-2}:
\begin{equation*}
\begin{aligned}
& \min_{g} & & \sum_{i \in [n]} \lambda_i g_i \\
& \text{subject to}
& & g_i \ge r^\cD_i & \forall i \in [n],\\
& & & g_i \ge \sum_{j \in [n]} \rho_{ij} g_j & \forall i \in \overline{U}.
\end{aligned}
\end{equation*}
Similar to \eqref{lp:mc-tao}, $g_i$ denotes the expected revenue obtained when the customer visits state $i$. If $i \in U$, then we may include $i$ in $S$ without loss of optimality, so $g_i = r^\cD_i$ and $i$ is present in the first set of constraints but not the second. Otherwise, $g_i = \max\{r^\cD_i, \sum_{i \in [n]} \rho_{ij} g_j\}$, i.e., the expected revenue $g_i$ is either $r^\cD_i$ from state $i$ or $\sum_{j \in [n]} \rho_{ij} g_j$ by transitioning to other states $j \in [n]$. Let $g^*$ be the optimal solution, then offering $\{i \mid g^*_i = r^\cD_i\}$ maximizes the expected revenue.
\end{proof}

\end{document}